\pdfoutput=1

\documentclass[journal]{IEEEtran}

\usepackage[T1]{fontenc}
\usepackage[protrusion=true,expansion=false,final]{microtype}
\usepackage{balance}
\usepackage{amsmath,amssymb,amsfonts}
\usepackage{amsthm}
\usepackage[varg]{newtxmath}
\usepackage{graphicx}
\usepackage{array}
\usepackage{booktabs}
\usepackage{multirow}
\usepackage{url}
\usepackage{xcolor}
\usepackage{tikz}
\usetikzlibrary{arrows.meta,positioning,calc}
\usepackage[ruled,vlined,linesnumbered]{algorithm2e}
\SetAlgoCaptionLayout{textbf}
\SetKwComment{Comment}{$\triangleright$\ }{}
\SetKwInOut{Input}{\textbf{Input}}
\SetKwInOut{Output}{\textbf{Output}}
\usepackage[colorlinks=true,linkcolor=blue,citecolor=blue,urlcolor=blue,
            filecolor=blue]{hyperref}
\usepackage{float}

\newcommand{\includefig}[2][]{%
  \includegraphics[#1]{#2}%
}

\newcommand{\proc}[1]{\textnormal{\textsc{#1}}}

\newcommand{\hyb}{-\hspace{0pt}}

\definecolor{cground}{HTML}{999999}
\definecolor{cair}{HTML}{0072B2}
\definecolor{cland}{HTML}{009E73}
\definecolor{cgate}{HTML}{D55E00}

\newtheorem{theorem}{Theorem}
\newtheorem{proposition}{Proposition}
\newtheorem{corollary}{Corollary}

\newcommand{\Fset}{\mathcal{F}}
\newcommand{\CL}{\mathcal{C}_L}

\newlength{\nomlabelwd}
\newlength{\nomlabelsep}
\newcommand{\nomsetwidest}[1]{%
  \settowidth{\dimen0}{#1}%
  \ifdim\dimen0>\nomlabelwd \nomlabelwd=\dimen0\fi
}
\newenvironment{nomenclature}{%
  \nomlabelwd=0pt
  \nomsetwidest{$K_\theta,K_\phi,K_d$}%
  \nomsetwidest{$\omega_{\theta0},\omega_{\phi0}$}%
  \nomsetwidest{$\delta,\delta_{\mathrm{roll}}^{\max}$}%
  \nomsetwidest{$\omega_{\max},\omega_{\min}$}%
  \nomsetwidest{$\omega_{\mathrm{target}}$}%
  \nomsetwidest{$v_{\max}(d)$}%
  \nomsetwidest{$v_{\mathrm{crit}},v_{\min}$}%
  \nomsetwidest{$z_{\mathrm{lm}},\bar\omega_w$}%
  \settowidth{\nomlabelsep}{mmmm}%
  \begin{list}{}{%
    \setlength{\labelwidth}{\nomlabelwd}%
    \setlength{\labelsep}{\nomlabelsep}%
    \setlength{\leftmargin}{\nomlabelwd}%
    \addtolength{\leftmargin}{\nomlabelsep}%
    \setlength{\itemsep}{0pt}%
    \setlength{\parsep}{0pt}%
    \setlength{\topsep}{0pt}%
  }%
}{%
  \end{list}%
}

\begin{document}

\title{DART: Dual-Axis Airborne Reachability-Gated Torque-Reaction
for Off-Road Vehicle Jumps}

\author{Yu Hu\textsuperscript{1,2,*}, Fangzhou Zhao\textsuperscript{1}, Mingyuan Sang\textsuperscript{1,2}, Chen Min\textsuperscript{1}, Liang Chen\textsuperscript{1}, Wei Li\textsuperscript{1,2},\\
Wenyu Kuang\textsuperscript{3}, Shican Chen\textsuperscript{3}, Jinwei Li\textsuperscript{3}, Baolei Chen\textsuperscript{3}
\thanks{\textsuperscript{1}Research Center for Intelligent Computing Systems, Institute of Computing Technology, CAS, Beijing 100190, China.}
\thanks{\textsuperscript{2}School of Computer Science and Technology, University of Chinese Academy of Sciences, Beijing 100049, China.}
\thanks{\textsuperscript{3}Dong Feng Off-Road Vehicle Co., Ltd, Shiyan 442000, China.}
\thanks{\textsuperscript{*}Corresponding author: Yu Hu (e-mail: \href{mailto:huyu@ict.ac.cn}{huyu@ict.ac.cn}).}}

\maketitle

\begin{abstract}
Traversing crests, ledges, and ditches at high speed often launches vehicles into the air, and a mishandled landing presents a substantial crash hazard. We show that the airborne phase is barely controllable: on a \mbox{1383\,kg} platform the wheel angular-momentum budget caps the recoverable pitch-rate change at roughly $9$--$13^\circ$/s in the tighter nose-up direction under drive at typical takeoff wheel speeds, and at about twice that in the reverse-inclusive braking direction; driving the wheels to their drivetrain hard limit raises the measured nose-up ceiling to only $16$--$18^\circ$/s. Takeoff pitch-rate disturbances beyond this directional budget are physically unrecoverable in flight, so the decisive leverage lies before takeoff. DART (Dual-Axis Airborne Reachability-Gated Torque-Reaction) back-propagates the landing constraint into a closed-form certified feasible-takeoff set, which supplies a conservative go/no-go condition and a pre-takeoff speed-shaping law. In flight, DART regulates pitch and roll via steer-resolved wheel-reaction torque, governed by a per-flight roll latch derived from the yaw-coupling analysis. In deterministic full-scale simulation in BeamNG.tech, a calibrated pre-takeoff speed regulator reduces touchdown speed by $36\%$ and raises on-target landings from 0/30 to 30/30. Under the same steep-lip approach the airborne law completes 29/30 safe landings under crash-avoidance bounds versus 0/30 for reaction-wheel-style PD (RW-PD) and time-optimal bang-bang (TOBB). On banked run-ups DART holds the median pitch error at or below $2^\circ$ at every cross-slope, with the largest baseline separation at $\gamma{=}12^\circ$. Across disturbance regimes, the latch preserves pitch-only allocation on low-disturbance entries and enables dual-axis control when roll becomes binding. All results are from simulation; hardware validation remains open.
\end{abstract}


\section*{Nomenclature}
\begin{nomenclature}
\item[$\theta,\phi,\psi$] Body pitch, roll, yaw.
\item[$\omega_\theta,\omega_\phi$] Body pitch and roll rates.
\item[$v_z$] Body-frame vertical velocity.
\item[$\theta_0,\phi_0,v_0$] Takeoff pitch, roll, and speed.
\item[$\omega_{\theta0},\omega_{\phi0}$] Takeoff pitch and roll rates.
\item[$\omega_{i,0}$] Takeoff spin speed of wheel $i$.
\item[$I_{yy},I_{xx}$] Body pitch and roll inertias.
\item[$I_i,\omega_i,\tau_i$] Wheel-$i$ inertia, spin rate, torque.
\item[$\omega_{\max},\omega_{\min}$] Wheel spin-rate bounds.
\item[$\tau_{\max}$] Per-wheel drive/braking torque bound.
\item[$r,m$] Wheel radius; vehicle mass.
\item[$n_w,n_f,n_r$] Driven, front-driven, rear-driven counts.
\item[$B$] Directional budget, in body-rate units.
\item[$B_\uparrow,B_\downarrow$] Nose-up drive and braking budgets.
\item[$B_r$] Direction-aware margin after nulling.
\item[$B_{\mathrm{eff}}$] Pitch budget under sustained roll.
\item[$B_{\mathrm{sym}}$] Zero-steer symmetric budget.
\item[$H_f,H_r$] Front-/rear-axle momentum budgets.
\item[$p_r$] Rear-axle pitch-only contribution.
\item[$T$] Airborne flight time.
\item[$T_c$] Correction time left after nulling.
\item[$a_\theta$] Pitch angular-acceleration bound.
\item[$e_0,r_\theta$] Initial rate error; residual pitch.
\item[$\Pi_\theta$] Projection onto $(\theta_0,\omega_{\theta0})$.
\item[$\CL$] Landing terminal constraint set.
\item[$\Fset^\star$] Exact backward-reachable takeoff set.
\item[$\Fset,\Fset_\theta$] Certified takeoff set; its pitch slice.
\item[$v_{\max}(d)$] Admissible speed at run-up distance $d$.
\item[$v_{\mathrm{crit}},v_{\min}$] Certified speed-window bounds.
\item[$\sigma$] Takeoff jitter (deg; m/s on speed).
\item[$\mathcal{B}_\rho$] Takeoff uncertainty ball, radius $\rho$.
\item[$\delta,\delta_{\mathrm{roll}}^{\max}$] Front steering angle; roll-steer saturation.
\item[$K_\theta,K_\phi,K_d$] Pitch, roll, and pitch-rate gains.
\item[$\phi_{\mathrm{db}},\phi_{\mathrm{on}}$] Roll deadband; latch threshold.
\item[$w_{\mathrm{roll}}$] Per-flight roll latch state.
\item[$\omega_{\theta,\max}$] Pitch-rate reference bound.
\item[$\tau_{\mathrm{spin}}$] Roll spin torque demand.
\item[$(\theta,\phi,\omega)_{\mathrm{target}}$] Touchdown targets: pitch, roll, rate.
\item[$n_{\mathrm{ph}},n^{\ast}$] No-contact samples: P0$\to$P1; P1 gating.
\item[$z_{\mathrm{lm}},\bar\omega_w$] Land-match height; spin-up guard.
\item[$\alpha,\beta,\gamma$] Ramp, landing, cross-slope angles.
\item[$\theta_{\mathrm{surf}}$] Landing-surface tangent at touchdown.
\item[$a_{\mathrm{brake}}$] Total decelerating specific force.
\end{nomenclature}

\section{Introduction}
\label{sec:intro}

\begin{figure*}[!t]
\centering
\includefig[width=\textwidth]{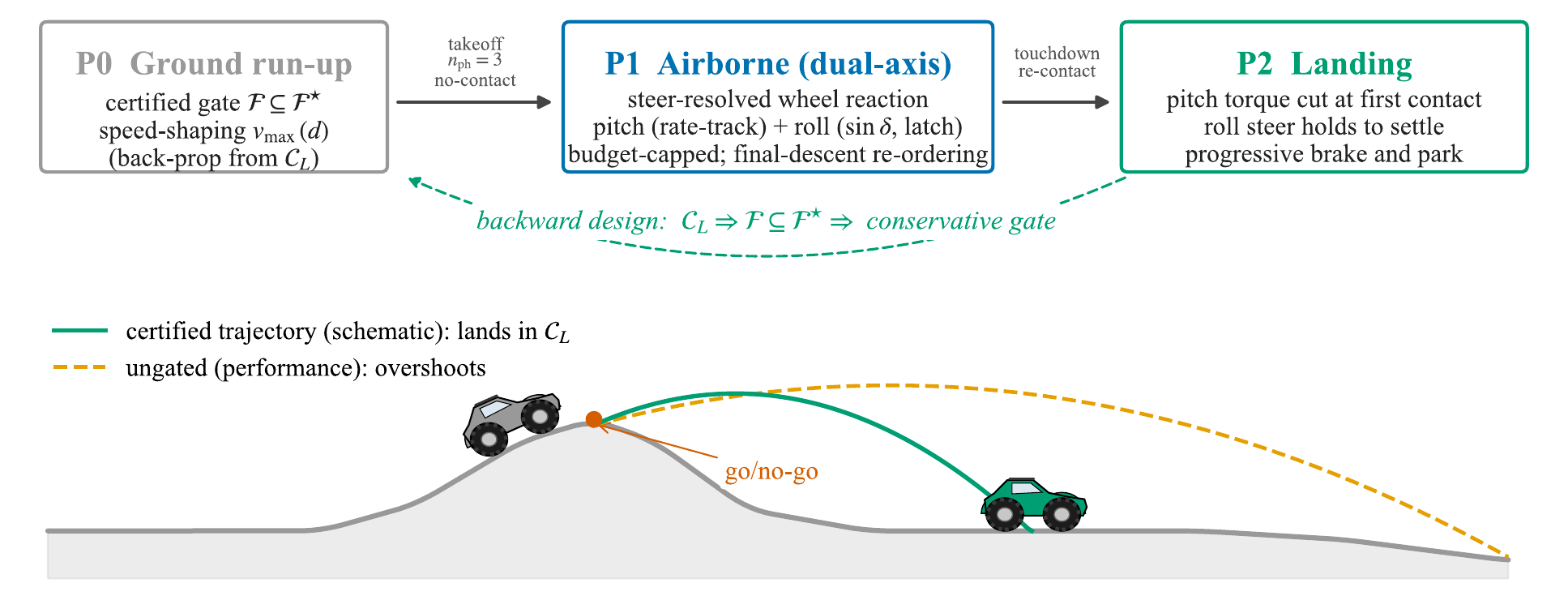}
\caption{DART overview. P0 gates and shapes the run-up using the certified
takeoff set, P1 regulates airborne pitch and roll and re-orders wheel speed
in the final descent, and P2 cuts pitch torque at first contact while roll
steer settles. The dashed green arrow denotes backward design from
$\CL$; the lower trajectories schematically contrast ungated and certified
takeoffs.}
\label{fig:overview}
\end{figure*}

Rally racing, rapid traversal, and time-critical response often require vehicles
to cross crests, ledges, and ditches without slowing to a crawl or detouring.
At high speed, a mishandled landing can cause a severe crash. The same terrain
feature poses little risk when traversed slowly, which may
explain why this high-speed failure mode remains comparatively underexplored.
During a ledge crossing, roughly one second of flight can contain the entire
safety margin because a pitch error of only a few degrees can turn a four-wheel
landing into a nose-first impact, rollover, or suspension bottom-out. Unlike
ground-phase maneuvers, the airborne phase offers no translational recourse
because, once the wheels leave the launch lip, the takeoff state fixes the
center-of-mass trajectory as a ballistic parabola. This leaves attitude as
the sole actuatable degree of freedom, accessible only through reaction
torques from accelerating or regeneratively braking individual wheels.

The deeper difficulty is not that the airborne phase is hard to control, but
that it is barely controllable. The wheel angular-momentum budget caps
recoverable body pitch-rate change at about $9$--$13^\circ$/s for a 1383\,kg
platform in the nose-up direction under drive, roughly one-tenth of
the motor-torque-rate figure; this is the direction that must arrest
the nose-down disturbances dominating steep takeoffs. The braking direction, extending through zero
into reverse, affords about twice this budget. No increase in
motor aggressiveness can undo a large takeoff pitch-rate disturbance once
airborne because wheel angular momentum, not motor torque, bounds the
available correction. This limit shifts the decisive leverage to the ground
phase, where the takeoff state can be shaped into the certified feasible set
$\Fset\subseteq\Fset^\star$. A verifiable sufficient condition can then
propagate the landing constraint backward, guide takeoff-state shaping, and
issue a conservative no-go when no certified state
remains within the run-up horizon.

Existing approaches fall into three main categories, each carrying a structural gap. In-air attitude controllers repurpose throttle and steering to reorient the body after launch; the closest examples are AGRO~\cite{gonzalez2020agro} and the recent in-air planner of Pokhrel \emph{et al.}~\cite{pokhrel2025dom}, yet neither derives the control-authority limit nor feeds the landing constraint back to the ground phase. Pre-takeoff trajectory and speed planners treat the airborne phase as an open-loop ballistic arc and encode jump risk only as a soft planning cost, rather than as a hard, verifiable feasibility gate. Driveline-protection methods deliberately disengage the drivetrain when a jump is detected and re-engage it at touchdown; in-flight wheel-speed modulation, where retained, serves only landing-speed synchronization, so the wheel-reaction authority for attitude control goes unused~\cite{us12233882}.

As shown in Fig.~\ref{fig:overview}, DART integrates pre-takeoff reachability
gating with airborne attitude control in a cross-phase framework. Before
launch, DART back-propagates the landing constraint into a certified
feasible-takeoff set that supplies a conservative go/no-go gate and a
pre-takeoff speed-shaping law. In flight, DART regulates pitch and roll via
steer-resolved wheel-reaction torque. We evaluate the integrated framework in
controlled full-scale simulation. Our contributions are as follows.

\begin{enumerate}
\item \textbf{Momentum-budget reachability theory for wheel-reaction airborne attitude.}
Beyond the $\sin\delta$/$\cos\delta$ Jacobian and stabilizing laws of prior
work, Theorem~\ref{thm:budget} establishes an impossibility bound.
Theorem~\ref{thm:frontier} derives a conservative certified two-axis
control-authority frontier. Theorem~\ref{thm:feasible} provides a
constructive closed-form feasibility certificate.
Corollary~\ref{cor:backprop} back-propagates
this certificate to yield the pre-takeoff speed-shaping law. Direction-resolved
tests locate the budget's binding onset where the closed-form formula
predicts it. This
reframes the problem from rescuing attitude in flight to shaping the takeoff
state into the certified set.

\item \textbf{Certified pre-takeoff gating and the integrated cross-phase chain.}
We deploy the feasibility certificate as a conservative go/no-go gate and a
calibrated one\hyb dimensional speed regulator, closing the loop from the landing
constraint back to ground-phase actuation. In end-to-end simulation on a
single full-scale vehicle, the regulator reduces touchdown speed by 36\% and
raises on-target landings from 0/30 to 30/30 in the evaluated
steep-lip scenario. When the certified speed window
is empty, the gate refuses the jump and stops the vehicle before it reaches
the lip.

\item \textbf{Regime-conditional dual-axis airborne control.}
DART regulates airborne pitch and roll through steer-resolved wheel-reaction
torque and uses a per-flight roll latch to implement the
allocation predicted by Theorem~\ref{thm:frontier}. The latch preserves
pitch-only behavior on \emph{benign} entries, those combining small takeoff
roll with a mild pitch rate, and activates dual-axis control when roll is
binding.
On a 1383\,kg platform in paired full-scale simulation, DART outperforms
reaction-wheel-style PD (RW-PD) on pitch across the approach experiments and
the stressed $\alpha{=}20^\circ$ air-impulse condition. The resulting regime map links
controller advantage to pitch-rate, combined pitch--roll, and yaw-coupled
disturbances.
\end{enumerate}

The remainder of the paper is organized as follows. Section~\ref{sec:related} reviews related work. Section~\ref{sec:method} develops the DART framework and its momentum-budget reachability theory. Section~\ref{sec:setup} describes the validation platform and protocol. Section~\ref{sec:results} proceeds from theory validation through decisive comparisons and generalization to regime synthesis. Section~\ref{sec:limits} discusses scale dependence, open issues, and sim-to-real gaps, and Section~\ref{sec:conclusion} concludes.

\section{Related Work}
\label{sec:related}

Research relevant to jump decisions and airborne reorientation spans
pre-takeoff planning and gating, and in-flight attitude control. We review them in
takeoff-to-air order and relate each line to the corresponding DART component.

\subsection{Pre-Takeoff Planning and Reachability Gating}
The most direct way to make a jump safe is to plan the takeoff trajectory and
speed beforehand, so that the resulting ballistic arc clears the terrain and
meets the landing surface. Sampling-based model-predictive path-integral
(MPPI) control~\cite{williams2016mppi,williams2018itmpc} plans the full
trajectory before takeoff, and terrain-aware variants predict the passive
ballistic segment and landing impact within the pre-takeoff
rollout~\cite{lee2023mppi}. Recent work learns, from terrain geometry, when
a jump is worthwhile~\cite{zhao2026jump}. The broader learning-based MPC
literature supplies the robustness guarantees~\cite{gandhi2021robustmppi}
and learned-dynamics models~\cite{hewing2020lbmpc} on which such planners
build. Two structural
properties distinguish this line from hard feasibility gating. The
planned air segment is an open-loop gravity arc with no active attitude
control, so the plan can only shape when and how fast the vehicle
leaves the lip, not how it rotates in flight. Moreover, jump risk typically
enters as a soft planning cost that competes against a progress reward inside
the sampling optimizer, whereas a reachability gate certifies constructive
feasibility by establishing $x_0\in\Fset$, which is a sufficient condition
for $x_0\in\Fset^\star$, or conservatively enforces the deceleration implied by
the terminal landing constraint. DART's gate addresses the complementary
case by shaping the requested takeoff speed to the landing geometry as a go/no-go
condition rather than an optimizer penalty.

Inspired by formal safety-certificate methods, we design DART's pre-takeoff
gate around the exact backward-reachable preimage $\Fset^\star$ and a
closed-form certified inner subset $\Fset\subseteq\Fset^\star$ for the airborne
double integrator under the wheel momentum budget. Hamilton-Jacobi reachability
computes backward-reachable sets of states from which a system can reach or
avoid a target despite bounded
disturbances~\cite{mitchell2005,bansal2017hj,chen2018hj}, while control barrier
functions enforce forward invariance of a safe set, typically through a
quadratic-program safety filter that minimally edits a nominal
input~\cite{ames2019cbf}. DART departs from both templates. Where the general
Hamilton--Jacobi program represents the reachable set by a grid-based value
function, this sufficient certificate is forward evaluable in closed form;
and where a barrier-function filter edits the in-flight input, the gate acts
before takeoff, where the bounded airborne authority can still be respected.

\begin{table*}[!tp]
\centering
\caption{Feature comparison of pre-takeoff planners and wheel-reaction airborne controllers.}
\label{tab:positioning}
\begin{tabular}{@{}lccccc@{}}
\toprule
 & MPPI & AGRO & DOM & Kim \emph{et al.} & DART \\
 & planners & \cite{gonzalez2020agro} & \cite{pokhrel2025dom} & \cite{kim2021comparing} & (ours) \\
\midrule
Airborne attitude control & --- & \checkmark & \checkmark & \checkmark & \checkmark \\
Dual-axis (pitch $+$ roll) & --- & \checkmark & \checkmark & \checkmark & \checkmark \\
Authority (momentum) bound & --- & --- & --- & --- & \checkmark \\
Closed-form certified takeoff set & --- & --- & --- & --- & \checkmark \\
Landing$\to$takeoff feedback & soft cost & --- & --- & --- & analytic sufficient gate \\
Mass regime & varies & RC & RC & idealized & 1383\,kg (sim) \\
\bottomrule
\end{tabular}
\end{table*}

\subsection{Momentum-Exchange Attitude Control}
Classical methods for reorienting a body in free fall by exchanging angular
momentum include spacecraft reaction wheels and control-moment
gyros~\cite{wertz1978}, reaction-wheel-stabilized jumping and balancing
cubes~\cite{gajamohan2012cubli}, inertial tails for mid-air attitude
correction~\cite{libby2012tail,changsiu2013tail}, and the falling-cat
reflex~\cite{kane1969cat}. Among legged platforms that perform jumps and
aerial maneuvers, the MIT Cheetah family first realized running
jumps over obstacles through online pre-takeoff planning
\cite{park2015jumps,nguyen2019cheetah}, then $360^\circ$ backflips via
offline trajectory optimization~\cite{katz2019minicheetah}, and finally
online centroidal-momentum trajectory optimization of in-air reorientation
maneuvers such as flips and barrel rolls~\cite{chignoli2021aerial}.
The wheeled-vehicle setting differs from these platforms in two mechanistic
ways. Where prior systems carry a dedicated momentum-exchange actuator (a
reaction wheel, gyro, or tail) or reconfigure large-range limbs (legged
platforms), DART repurposes the existing four-motor drive and
front-wheel steering as the exchange actuator, obtaining roll as a second controllable axis at no added mass from steer-induced drive
reaction, a capability a single fixed-axis reaction wheel cannot provide.
More consequential for the budget analysis is the recoverable attitude
itself. A spacecraft reaction wheel also has a finite speed envelope and
hence a bounded body-rate change, but the classical slew task is angle
reorientation without a hard time limit: a bounded rate integrated over an
unconstrained maneuver time yields effectively unbounded attitude
authority, and the wheels can be desaturated between slews. The
falling cat, inertial tails, and legged-limb maneuvers instead redistribute
rather than store momentum, so their recoverable attitude is set by
zero-net-momentum body reconfiguration. In a wheeled vehicle, however, the wheels saturate at a finite spin speed within the roughly one-second flight, so a hard \emph{momentum budget} caps the recoverable body rate. This deadline-bound, non-desaturable budget regime has no operative analogue in the spacecraft or tail literature and must be quantified rather than assumed sufficient.

\subsection{Wheel-Reaction Airborne Attitude Control}
The line closest in mechanism to DART's wheel-reaction law reorients an airborne ground vehicle using the
reaction torques of independently driven and steered wheels. On the analytic
side, Gonzalez \emph{et al.} (AGRO)~\cite{gonzalez2020agro} pioneered the use of
in-wheel-motor reaction torques for in-air pitch and roll control, and Kim
\emph{et al.}~\cite{kim2021comparing} provided the full aerial dynamics with
feedback-linearizing and adaptive-backstepping laws, Lyapunov stability
guarantees, and the $\sin\delta$/$\cos\delta$ torque Jacobian, establishing
feasibility in simulation on a small idealized platform. On the learning side, Pokhrel
\emph{et al.}~\cite{pokhrel2025dom} repurpose throttle and steering to control
airborne attitude. Their bicycle-model formulation yields the same
steer-resolved roll term, written $\sin\delta$ in our notation, and likewise
notes the absence of yaw authority. They denote steering angle by $\psi$,
whereas we reserve $\psi$ for yaw. Building on this model, they learn a hybrid physics-plus-data
kinodynamic model that reaches a prescribed landing pose by sampling
airborne controls over a fixed horizon, demonstrated on a gimbal rig and a
small RC car; this is the closest prior work that actively controls the air phase
toward a landing objective. A related
patent~\cite{us11981873} discloses distributed drive/brake torque for in-air pitch
and roll in a reactive scheme.

DART differs from prior work in three respects. It provides (i)~an explicit
angular-momentum budget and closed-form certified
feasible set (Section~\ref{sec:theory}), rather than an assumed reachable horizon;
(ii)~back-propagation of the landing constraint into a
pre-takeoff reachability gate, whereas prior approaches evaluate the air phase
largely in isolation; and (iii)~landing-outcome and geometry-generalization
tests at full vehicle mass (1383\,kg). Table~\ref{tab:positioning} summarizes
the comparison.

\section{The DART Cross-Phase Framework}
\label{sec:method}

DART is organized as a chain of three phases (ground run-up P0, airborne P1,
and landing P2) coupled solely through the landing
terminal constraint $\CL$ and the certified feasible takeoff set $\Fset$.
The design proceeds backward, with Section~\ref{sec:backprop} deriving $\Fset$
from $\CL$ via the airborne dynamics, Section~\ref{sec:shaping} using $\Fset$
to shape P0, and Sections~\ref{sec:dual}--\ref{sec:alloc} executing the
in-flight control.

\subsection{Problem Statement}
A jump comprises three phases, a ground run-up (P0), an airborne phase (P1),
and landing (P2), with phase switches triggered by $n_{\mathrm{ph}}$
consecutive
zero-contact samples (P0$\rightarrow$P1) and by ground re-contact
(P1$\rightarrow$P2). Among these three phases, landing is the decisive safety
event because the touchdown attitude must match the
landing surface, and the angular and vertical rates must fall inside the
suspension absorption window. We pose jump safety as a terminal
constraint on the P2 touchdown state,
\begin{multline}
\CL = \{ (\theta,\phi,\omega_\theta,\omega_\phi,v_z) :
|\theta-\theta_{\mathrm{surf}}|\le\bar\theta,\,|\phi|\le\bar\phi,\\
|\omega_\theta|\le\bar\omega_\theta,\,|\omega_\phi|\le\bar\omega_\phi,\,
|v_z|\le\bar v_z \}
\end{multline}
where $\omega_\theta$ and $\omega_\phi$ are body pitch and roll rates,
$v_z$ is body-frame vertical velocity, $\theta_{\mathrm{surf}}$ is the
landing-surface tangent at the touchdown point,
$(\bar\theta,\bar\phi,\bar\omega_\theta,\bar\omega_\phi,\bar v_z)$ are the
suspension-absorption tolerances. Attitudes are expressed in a body frame
($x$ forward, $y$ left, $z$ up), with roll $\phi$ (about $x$) and yaw $\psi$
(about $z$) following the right-hand convention, and pitch $\theta$
nose-up positive, the right-handed rotation about $-y$.

The control problem is to drive the P2 touchdown state into $\CL$. Because yaw
is unactuated in P1, $\CL$ excludes $\psi$ and retains it as an uncontrolled
closed-loop outcome.
P1 offers no translational control and only bounded attitude authority. We
derive that bound next, so membership in the exact backward-reachable set
$\Fset^\star$ is necessary. The gate enforces the stronger, sufficient condition
$x_0\in\Fset\subseteq\Fset^\star$ constructed in
Section~\ref{sec:backprop}.

\subsection{Airborne Dynamics and Momentum Budget}
\label{sec:budget}
In flight the only external force is gravity, which acts at the center of
mass and exerts no torque about it. Total angular momentum of the
chassis-plus-wheels system is conserved; any change in body angular momentum
is drawn from the wheels:
\begin{multline}
H_{\mathrm{body}}(t)+\sum_i I_i\,\omega_i(t)\,\hat e_i=\text{const} \\
\Longrightarrow\;
I\,\dot\omega_{\mathrm{body}}=-\sum_i I_i\,\dot\omega_i\,\hat e_i
\label{eq:conservation}
\end{multline}
where $H_{\mathrm{body}}$ and $\omega_{\mathrm{body}}$ denote body angular
momentum and rate, $I$ is the body inertia tensor, and $I_i$, $\omega_i$, and
$\hat e_i$ are the inertia, spin rate, and spin-axis unit vector of wheel~$i$.
The planar budget analysis below scalarizes the body rate to its pitch
component $\omega_\theta$.
The differentiated form holds exactly for the planar pitch motion analyzed
below ($\delta=0$, all spin axes parallel to the body $y$-axis); general
three-dimensional motion adds the transport term
$\omega_{\mathrm{body}}\times H_{\mathrm{body}}$ and steering-induced
$\dot{\hat e}_i$ terms, which vanish in that planar case.
Accelerating a wheel under drive pitches the nose up; regenerative braking pitches
it down. This gives a bidirectional pitch authority that single-sided
throttle/brake reactive control lacks.

Steering resolves this reaction onto two axes. A wheel steered by
$\delta$ has spin axis $\hat e_i=(-\sin\delta,\cos\delta,0)$ in the body frame;
its reaction torque splits into a pitch component $\propto\cos\delta$
and a roll component $\propto\sin\delta$. Rear wheels contribute only to
pitch, whereas the steered front wheels provide the $\sin\delta$ roll
component.
This resolution is the same Jacobian as in prior work~\cite{gonzalez2020agro,pokhrel2025dom,kim2021comparing};
we restate it to keep the budget analysis self-contained. The two channels are
co-allocated, so large roll demands borrow front-axle authority from pitch.

This bidirectional, two-axis authority is bounded by two distinct limits. We illustrate their scale using the nominal platform parameters of
Section~\ref{sec:setup}. These are $m\approx1383$\,kg and
$I_{yy}\approx2043$\,kg\,m$^2$, with
four wheels of $I_w=1.2$\,kg\,m$^2$ and $r=0.36$\,m, an effective per-wheel
drive torque of ${\approx}1200$\,N\,m and a sustained wheel-speed envelope
$\omega_{\max}\approx1200$\,rpm, both trace-derived at the in-flight
operating point. The drivetrain's design limits are higher, but the motor
torque curve decays over the wheel speeds visited during airborne
corrections, so these effective values are what the budget consumes; a
hard-limit calibration in Section~\ref{sec:res-reachf1} measures the
headroom that remains above them.
Instantaneous torque sets how fast the body rate can change. In flight, the
chassis reaction to each motor torque is the torque itself, so
$\dot\omega_{\mathrm{body}}\le\sum_i\tau_i/I_{yy}\approx2.35$\,rad/s$^2$,
which corresponds to a rate authority of ${\approx}135^\circ$/s over a
1\,s flight;
wheel-speed saturation instead limits each wheel to a finite
$\Delta\omega_i$ and caps the total recoverable body-rate change at
\begin{equation}
\Delta\omega_{\mathrm{body}}\le
\frac{\sum_i I_i\,\Delta\omega_i^{\max}}{I_{yy}}
\label{eq:budget}
\end{equation}
where $\Delta\omega_{\mathrm{body}}$ is the recoverable body pitch-rate change,
$\Delta\omega_i^{\max}=\omega_{\max}-\omega_{i,0}$ is the remaining
spin-up margin of wheel~$i$, and $I_{yy}$ is the body pitch inertia.
Equation~\eqref{eq:budget} illustrates the nose-up drive direction;
Theorem~\ref{thm:budget} gives the corresponding nose-up and nose-down
budgets $B_\uparrow$ and $B_\downarrow$. With the
in-flight spin-up margin $\Delta\omega_i^{\max}\approx70$--$95$\,rad/s at
typical takeoff wheel speeds, this caps the recoverable
nose-up body pitch-rate change at about 0.16--0.23\,rad/s
($\approx9$--$13^\circ$/s), roughly one-tenth of the torque-rate figure; the
braking budget $B_\downarrow$, extending through zero into reverse rotation,
is about twice as large at the same wheel speeds.

\subsection{Landing-Constraint Back-Propagation}
\label{sec:backprop}
Given $\CL$ at the predicted landing surface and the airborne dynamics of
Section~\ref{sec:budget} under the torque and budget limits, we define the exact
backward-reachable set
\begin{multline}
\Fset^\star=\{x_0=(\theta_0,\omega_{\theta0},\phi_0,\omega_{\phi0},v_0):\\
\exists\,u(\cdot)\ \text{s.t.}\ \mathrm{P1}(x_0,u)\ \text{ends in}\ \CL\}
\end{multline}
where $x_0=(\theta_0,\omega_{\theta0},\phi_0,\omega_{\phi0},v_0)$ is the
takeoff state and $u(\cdot)$ is an admissible in-flight control input
trajectory. In words, $\Fset^\star$ contains every takeoff state from which some
admissible in-flight control sequence drives the vehicle into the terminal
window. DART uses a constructive certificate
$\Fset\subseteq\Fset^\star$. Section~\ref{sec:theory} certifies its pitch
slice $\Fset_\theta\subseteq\Pi_\theta(\Fset^\star)$ through the
null-then-correct policy of Theorem~\ref{thm:feasible}, its roll slice through
Corollary~\ref{cor:rollaxis}, and its ballistic speed condition through
Proposition~\ref{prop:env}. Membership in $\Fset$ is sufficient for
feasibility but need not exhaust all admissible control sequences. Because translation is
ballistic, flight time $T$ and landing-surface tangent $\theta_{\mathrm{surf}}$
are functions of $(v_0,\theta_0,\text{geometry})$; the pitch sub-problem then
reduces to checking whether the bounded-budget rate authority can null the
residual $|\theta(T)-\theta_{\mathrm{surf}}|$ and $\omega_\theta(T)$.
Using the directional budgets of Theorem~\ref{thm:budget} yields a necessary
terminal-rate condition; the null-then-correct construction adds a sufficient
attitude-displacement condition for the certified set $\Fset$. The two
remaining coordinates of $\Fset$ are certified analogously, with the roll
pair $(\phi_0,\omega_{\phi0})$ using the corresponding closed-form certificate under the
steer-resolved budget given by Corollary~\ref{cor:rollaxis} in
Section~\ref{sec:theory}. The takeoff speed $v_0$ determines the
touchdown vertical rate through the ballistic arc, so the $\bar v_z$
condition of $\CL$ becomes the speed window of Proposition~\ref{prop:env}.

$\Fset$ is verifiable because it is a forward-evaluable analytic condition
rather than a learned estimate. This verifiability is precisely what allows
$\Fset$ to serve as a hard gate, in contrast to a soft planning cost.
Throughout, we keep the budget criterion separate from the speed-shaping
mechanism. \emph{Speed shaping} matches the takeoff speed to the landing
geometry, whereas the \emph{budget} is the angular-momentum authority limit.
The two are coupled because both are driven by the takeoff
speed through the wheel spin state, as Corollary~\ref{cor:asym} shows.

\subsection{Receding-Horizon Takeoff-Speed Shaping}
\label{sec:shaping}
We propagate $\Fset$ backward through the P0 run-up. The ground controller
solves a receding-horizon problem that keeps the predicted takeoff state inside
$\Fset$, primarily through speed shaping. With remaining run-up distance $d$
to the lip and a critical takeoff speed $v_{\mathrm{crit}}$, defined as the
boundary of $\Fset$ projected onto speed for the current geometry, the
admissible speed envelope is
\begin{equation}
v_{\max}(d)=\sqrt{v_{\mathrm{crit}}^2+2\textstyle\int_0^{d} a_{\mathrm{brake}}(s)\,ds}
\label{eq:vmaxint}
\end{equation}
where $a_{\mathrm{brake}}(s)>0$ is the total decelerating specific force at
run-up position $s$, namely the resultant of service/regenerative braking,
rolling and aerodynamic resistance, and the along-track gravity component,
all of which vary along the arc. The implementation uses the constant-force
reduction $v_{\max}(d)=\sqrt{v_{\mathrm{crit}}^2+2a_{\mathrm{brake}}d}$ with
online adaptation absorbing the neglected variation.
Exceeding this bound triggers physical braking, while staying below it allows
throttle tracking up to $v_{\max}(d)$. The speed-shaping law acts as a
closed-loop speed regulator throughout the run-up; the reachability gate
retains the binary go/no-go decision.

\subsection{The Go/No-Go Takeoff Decision}
\label{sec:go-nogo}
Before takeoff, the gate issues a go or no-go decision. If the predicted takeoff state
lies outside $\Fset$ and cannot be corrected within the remaining run-up, the
system executes a safe action rather than committing to a jump without a
feasibility certificate. It either aborts the jump or decelerates and
traverses the terrain feature slowly. This replaces learned extrapolation
beyond the training distribution with a verifiable analytic boundary. We demonstrate both
branches in closed loop.

\subsection{Dual-Axis Airborne Attitude Control Law}
\label{sec:dual}
Our airborne law regulates pitch and roll through \emph{torque-reaction}
actuation. We use the driven wheels as a reaction-wheel-like momentum
actuator. A common all-wheel drive torque changes the total wheel angular
momentum and produces a pitch reaction; the front-wheel steer angle
$\delta$ rotates that reaction vector and opens the roll axis
($\propto\sin\delta$). Unlike conventional torque vectoring, which
distributes torque unevenly across wheels for ground yaw control, DART
applies a common drive torque in flight; the attitude authority arises from
the body's reaction to the total wheel torque.

Because wheels are angular-momentum-bounded actuators, we use
saturated rate tracking rather than direct angle-PD. The rate-reference
bound caps the momentum the tracking loop may request, a guarantee that a
direct angle-PD gain does not provide. We
generate a saturated pitch-rate reference from the angle
error,
\begin{equation}
\omega_\theta^{\mathrm{des}}=\mathrm{clip}\big(K_\theta(\theta_{\mathrm{target}}-\theta),\,
\pm\omega_{\theta,\max}\big)
\end{equation}
where $\theta_{\mathrm{target}}$ is the target pitch at touchdown, $K_\theta$ the
proportional gain, and $\omega_{\theta,\max}$ the pitch-rate reference bound;
the reference converges to zero as
$\theta\to\theta_{\mathrm{target}}$. We apply a common motor-torque term to all four wheels
from the rate error $e_\omega=\omega_\theta^{\mathrm{des}}-\omega_\theta$ via
$\Delta\tau_{\mathrm{pitch}}=K_d e_\omega$, where $K_d$ is the pitch-rate
derivative gain; the body pitch reaction is the resulting rate of change of
total wheel angular momentum, as Eq.~\eqref{eq:conservation} shows, not a front-versus-rear
torque split. The implementation realizes $K_d$ as a normalized drive gain
that maps the rate error to a unit-saturated drive command, and the motor
envelope supplies the torque magnitude.
The resulting per-wheel commands saturate at the motor-torque
envelope $|\tau_i|\le\tau_{\max}$ through the final clip in
Algorithm~\ref{alg:dart};
this is the same per-wheel bound the reduced model carries through
$a_\theta=n_w\tau_{\max}/I_{yy}$.

For roll, we use steer-induced drive reaction and set the front-wheel steering
angle via
\begin{equation}
\label{eq:rolllaw}
\delta=w_{\mathrm{roll}}\cdot
\mathrm{clip}(K_\phi(\phi_{\mathrm{target}}-\phi),\,
\pm\delta_{\mathrm{roll}}^{\max})
\end{equation}
where $\phi_{\mathrm{target}}$ is the target roll at touchdown, $K_\phi$ the roll gain,
$\delta_{\mathrm{roll}}^{\max}$ the roll steer saturation, and
$w_{\mathrm{roll}}\in\{0,1\}$ is the per-flight roll latch defined in
Eq.~\eqref{eq:wroll}; as in Algorithm~\ref{alg:dart}, the roll channel is
additionally gated to $\delta=0$ whenever the roll error is inside the
roll deadband.
We use a common drive torque to keep the wheels spinning, so the reaction
carries a roll component $\propto\sin\delta$. We saturate roll steer
independently so that the small steering clamp applied during ground contact
does not curtail the in-flight roll authority.

The P0$\to$P1 switch requires a short no-contact count together with
lip-proximity and apex-height guards that reject transient wheel unloading
on the ramp face. Torque engagement requires a longer count, which keeps
reaction torque from disturbing the takeoff state, and a wheel-speed margin
consistent with the momentum budget of Theorem~\ref{thm:budget}.
Algorithm~\ref{alg:dart} folds these conditions into two guard predicates:
\proc{TakeoffConfirmed} debounces the no-contact count and adds the
lip-clearance guards for the phase switch, while \proc{EngageOK} debounces
torque engagement and adds the wheel-speed margin. Section~\ref{sec:setup}
lists the thresholds.

When roll error stays below the roll deadband, our controller leaves the roll
channel inactive ($\delta\to 0$), allocates the front-axle reaction budget
entirely to pitch, and reduces to \emph{pitch-only rate
tracking}; this is distinct from the reactive angle-PD baseline, which uses a
different pitch map and a separate counter-steer roll loop. Between the
deadband and full engagement, an
always-on-roll policy would keep the roll channel engaged whenever
$|\phi_{\mathrm{err}}|$ exceeds the deadband and would route every attitude
correction through the body-horizontal actuator plane.

The wheel-spin and steering inputs generate reaction torque only in the
body-horizontal plane, so they provide no direct yaw authority about the
body-$z$ axis. When the vehicle is tilted, however, this plane is no longer
horizontal in the world frame, and part of the reaction torque projects onto
the world-vertical axis:
\begin{equation}
\label{eq:yawleak}
M_z^{\mathrm{world}}
=\big(R(\phi,\theta)\,I_w\dot\omega_w\,\hat a(\delta)\big)_z
\;\propto\;\sin(\mathrm{tilt})\cdot\tau_{\mathrm{applied}}
\end{equation}
where $\dot\omega_w$ is the common wheel spin acceleration generated by the
drive reaction,
$\hat a(\delta)=[-\sin\delta,\cos\delta,0]^\top$ the steered spin axis,
$I_w$ a representative wheel inertia, $\tau_{\mathrm{applied}}$ the applied
motor torque, $R(\phi,\theta)$ the body-to-world rotation, and
$M_z^{\mathrm{world}}$ the world-frame yaw torque of the wheel assembly;
by the conservation of Eq.~\eqref{eq:conservation}, the body receives the
negative of this wheel-side projection. The projected fraction
scales with $\sin(\mathrm{tilt})$; for example,
$\sin 28^\circ\approx0.47$, close to one half.

This projection produces yaw torque without necessarily causing net yaw
drift. If the correction torque alternates symmetrically, the resulting yaw
impulses largely cancel. Persistent drift arises when the correction contains
a sustained same-sign component, either because the roll error remains biased
or because actuator latency disrupts this cancellation. The finite wheel-speed
margin limits the duration of such same-sign momentum exchange because a wheel
cannot continue supplying torque in that direction after reaching its
directional speed bound, thereby bounding the associated yaw impulse.

Net yaw drift therefore depends on how long roll correction keeps
$\delta\neq0$ while body tilt remains. Roll steering sustains the projection
in Eq.~\eqref{eq:yawleak}, but it is also the only in-flight means of reducing
the tilt that amplifies that projection. Any schedule that weakens or
intermittently disables roll while tilt persists prolongs the leakage
interval. We therefore use the per-flight latch
\begin{equation}
\label{eq:wroll}
w_{\mathrm{roll}}=
\begin{cases}
1 & \text{rest of flight, once } |\phi_{\mathrm{err}}|\ge\phi_{\mathrm{on}},\\[2pt]
0 & \text{otherwise, if never engaged this flight}
\end{cases}
\end{equation}
where $\phi_{\mathrm{err}}=\phi_{\mathrm{target}}-\phi$ and $\phi_{\mathrm{on}}$
is the roll-engagement threshold, set to $8^\circ$ in
Section~\ref{sec:setup}. The latch is re-initialized once per jump during
the ground phase rather than cleared on contact, because transient contact
flicker at launch would otherwise re-arm the channel mid-flight. It makes a
single within-flight commitment in which the
roll channel remains inactive if the error never reaches the threshold but
stays at full authority until landing once the threshold is crossed.
This avoids unnecessary steering on benign entries while correcting tilt as
rapidly as possible on stressed ones.

We reject three softer scheduling alternatives because each prolongs the
interval over which body tilt and steering coexist. A yaw-budget schedule
with $w=\max(0,1-|\psi_{\mathrm{air}}|/\psi_{\mathrm{budget}})$ reduces roll
authority only after yaw has already accumulated, leaving the tilt that
generates the projection in Eq.~\eqref{eq:yawleak}. A benefit-proportional
schedule with $w=\min(1,|\phi_{\mathrm{err}}|/\phi_{\mathrm{auth}})$ weakens roll
correction while the error is still present, extending the leakage
interval. Deadband-exit hysteresis repeatedly disengages and re-engages the
channel near the threshold; under pitch-heavy workloads, the resulting
correction bursts can accumulate yaw with the same effective sign. Our
per-flight latch avoids these failure modes through monotonic within-flight
allocation. On entries that exceed $\phi_{\mathrm{on}}$ immediately, it is
identical to always-on roll control by construction. The resulting yaw
mechanism and the latch-versus-always-on comparison are evaluated in closed
loop.

The implemented steer channel also carries a bounded heading trim acting
through the same tilted-plane projection of Eq.~\eqref{eq:yawleak}. Once
the in-air yaw error leaves a small deadband, a capped counter-steer
component is blended into $\delta$, and roll steer is attenuated at large
yaw error so that heading recovery is not paid for with additional tilt.
The trim adds no yaw authority beyond this projection and leaves the
pitch-axis budget untouched; touchdown yaw is still analyzed as a
disturbance in Section~\ref{sec:results}. Section~\ref{sec:setup} lists its
parameters.

\subsection{Three-Phase Torque Allocation}
\label{sec:alloc}
Per-wheel drive torques carry one common term that serves both the
pitch-drive demand $\Delta\tau_{\mathrm{pitch}}$ and, when roll is engaged,
the spin-hold demand $\tau_{\mathrm{spin}}$. The two demands act on the same
wheels, so the allocator applies the same-sign larger of the two rather
than their sum, and front steer $\delta$ resolves part of the front-wheel
reaction into roll. Thus pitch and roll share the front-wheel reaction
budget, the sharing that Theorem~\ref{thm:frontier} prices. Unlike
driveline-disengagement protection~\cite{us12233882},
DART keeps the driveline engaged as the airborne attitude actuator and, in
the final descent below the land-match height $z_{\mathrm{lm}}$, re-orders
wheel speed toward ground speed as attitude-neutral preparation for
touchdown.
Algorithm~\ref{alg:dart} summarizes the per-step cross-phase DART
controller. Each jump starts with $\mathrm{phase}=\mathrm{P0}$ and
$w_{\mathrm{roll}}=0$; contact does not clear the latch.

The window $[v_{\min},v_{\mathrm{crit}}]$ consumed by P0 is the deployed
one\hyb dimensional reduction of the certified projection $\Pi_v(\Fset)$. Its
upper bound $v_{\mathrm{crit}}$ is calibrated offline for each jump
geometry, its lower bound $v_{\min}$ is the clearance requirement of
Proposition~\ref{prop:env}, and an empty window triggers the no-go branch
demonstrated in Section~\ref{sec:res-abort}. The predicates
\proc{TakeoffConfirmed} and \proc{EngageOK} add lip-clearance and
wheel-speed-margin guards to the two debounce counts of
Section~\ref{sec:dual}; Section~\ref{sec:setup} specifies their settings.

\begin{algorithm}[!t]
\footnotesize
\DontPrintSemicolon
\caption{DART cross-phase control step.}
\label{alg:dart}
\Input{state $x=(\theta,\phi,\psi,\omega_\theta,\omega_\phi,v,\dots)$;
wheel speed $\omega_w$, radius $r$; height $z$ and land-match height
$z_{\mathrm{lm}}$; contact flag $c$; phase $\in\{\mathrm{P0,P1,P2}\}$;
run-up distance $d$; calibrated certified window
$[v_{\min},v_{\mathrm{crit}}]$ (Section~\ref{sec:setup}); speed request
$v_{\mathrm{req}}$; targets
$\theta_{\mathrm{target}},\phi_{\mathrm{target}}$; gains
$K_\theta,K_\phi,K_d$; bounds
$\omega_{\theta,\max},\delta_{\mathrm{roll}}^{\max}$ and motor-torque limits;
deadband $\phi_{\mathrm{db}}$; latch threshold $\phi_{\mathrm{on}}$ and state
$w_{\mathrm{roll}}$; braking bound $a_{\mathrm{brake}}$; roll spin demand
$\tau_{\mathrm{spin}}$}
\Output{updated phase and $w_{\mathrm{roll}}$; per-wheel torques
$\{\tau_i\}$ and front steer $\delta$}
\lIf{$c=1$ \textbf{and} $\mathrm{phase}=\mathrm{P1}$}{$\mathrm{phase}\gets\mathrm{P2}$}
\lIf{$c=0$ \textbf{and} $\mathrm{phase}=\mathrm{P0}$ \textbf{and}
\proc{TakeoffConfirmed}}{$\mathrm{phase}\gets\mathrm{P1}$}
$\Delta\tau_{\mathrm{pitch}}\gets0$,\ $\tau_{\mathrm{sp}}\gets0$,
$\{\tau_i\}\gets0$,\ $\delta\gets0$\;
\uIf{$\mathrm{phase}=\mathrm{P0}$}{
  \uIf{$[v_{\min},v_{\mathrm{crit}}]=\varnothing$}{
    $\{\tau_i\}\gets\proc{SafeNoGo}(x)$\tcp*[r]{abort or decelerate-through}
  }
  \Else{
  $v_{\max}(d)\gets\sqrt{v_{\mathrm{crit}}^2+2\int_0^{d}a_{\mathrm{brake}}(s)\,ds}$\tcp*[r]{Eq.~\eqref{eq:vmaxint}}
  $\{\tau_i\}\gets\proc{GroundSpeedControl}(v,\min\{v_{\max}(d),v_{\mathrm{req}}\})$\;
  }
}
\uElseIf{$\mathrm{phase}=\mathrm{P1}$ \textbf{and} $z>z_{\mathrm{lm}}$}{
  \tcp{P1a: airborne attitude correction}
  $\phi_{\mathrm{err}}\gets\phi_{\mathrm{target}}-\phi$\;
  \lIf{$|\phi_{\mathrm{err}}|\ge\phi_{\mathrm{on}}$}{$w_{\mathrm{roll}}\gets1$}
  \lIf{$|\phi_{\mathrm{err}}|>\phi_{\mathrm{db}}$}{$\delta\gets
  w_{\mathrm{roll}}\cdot\mathrm{clip}(K_\phi\phi_{\mathrm{err}},
  \pm\delta_{\mathrm{roll}}^{\max})$}
  \If{\proc{EngageOK}}{
    $\omega_\theta^{\mathrm{des}}\gets\mathrm{clip}\!\big(K_\theta(\theta_{\mathrm{target}}-\theta),\,\pm\omega_{\theta,\max}\big)$\;
    $\Delta\tau_{\mathrm{pitch}}\gets
    K_d(\omega_\theta^{\mathrm{des}}-\omega_\theta)$\;
    \lIf{$\delta\neq0$}{$\tau_{\mathrm{sp}}\gets\tau_{\mathrm{spin}}$}
    $\{\tau_i\}\gets\mathrm{sgn}(\Delta\tau_{\mathrm{pitch}})\,
    \max\{|\Delta\tau_{\mathrm{pitch}}|,\tau_{\mathrm{sp}}\}$ on all four
    wheels\tcp*[r]{shared reaction (Thm.~\ref{thm:frontier})}
  }
}
\uElseIf{$\mathrm{phase}=\mathrm{P1}$}{
  $\{\tau_i\}\gets\proc{WheelSpeedMatch}(\omega_w\!\to\!v/r)$\tcp*[r]{P1b ($z\le z_{\mathrm{lm}}$)}
}
\ElseIf{$\mathrm{phase}=\mathrm{P2}$}{
  \tcp{pitch differential off; roll steer holds}
  \lIf{$|\phi_{\mathrm{target}}-\phi|>\phi_{\mathrm{db}}$}{$\delta\gets
  w_{\mathrm{roll}}\cdot\mathrm{clip}(K_\phi(\phi_{\mathrm{target}}-\phi),\pm\delta_{\mathrm{roll}}^{\max})$}
}
clip $\{\tau_i\}$ and $\delta$ to the motor/steer limits\;
\Return $\mathrm{phase},w_{\mathrm{roll}},\{\tau_i\},\delta$\;
\end{algorithm}

\subsection{Reachability Theory and Envelopes}
\label{sec:theory}

The preceding subsections described the mechanism; here we formalize its
theoretical foundations. These results distinguish DART from prior wheel-reaction
attitude controllers, which give the $\sin\delta$/$\cos\delta$ Jacobian and a
stabilizing law but no reachability or authority theory.

All results in this section are stated for a nominal
reduced model and validated empirically on the full soft-body multibody
simulator; actuator--landing-surface coupling remains outside the reduced
model. We assume:
\begin{itemize}
    \item \textbf{A1:} Rigid-body pitch modeled as a double integrator driven by wheel-reaction torque; for Theorem~\ref{thm:frontier} and Corollary~\ref{cor:camber}, roll is modeled the same way.
    \item \textbf{A2:} Conservation of chassis-plus-wheel angular momentum in flight, since gravity exerts no moment about the center of mass.
    \item \textbf{A3:} Point-mass ballistic translation.
    \item \textbf{A4:} Each wheel speed remains within
    $[\omega_{\min},\omega_{\max}]$, giving finite drive and braking momentum increments.
    Here $\omega_{\min}<0$. The envelope includes reverse rotation, in which
    the wheels are driven through zero into reverse, and on the evaluated
    platform the sustained bounds are approximately symmetric,
    $\omega_{\min}\approx-\omega_{\max}$. At the in-flight operating point,
    Section~\ref{sec:setup} gives $[-1200,1200]$\,rpm.
    \item \textbf{A5:} For the speed--authority corollary only, the wheels leave the lip without longitudinal slip, so $\omega_{i,0}=v_0/r$.
\end{itemize}
Let $I_{yy}$ be the body pitch inertia, $I_i$ and $\omega_i$ the wheel inertias
and speeds, $\theta_{\mathrm{surf}}$ the terminal landing pitch defined by the
landing-surface tangent in $\CL$, $\omega_{\mathrm{target}}$ the
terminal target pitch-rate, and $T$ the flight time. Let $n_w$, $n_f$, and $n_r$ be
the numbers of driven, front driven, and rear driven wheels. Here $n_w=4$
and $n_f=n_r=2$. Let
$a_\theta=n_w\tau_{\max}/I_{yy}$ the symmetric pitch angular-acceleration bound used by the
reduced model, where $\tau_{\max}$ denotes the effective per-wheel
drive/braking torque bound at the in-flight operating point. Its value is
approximately $1200$\,N\,m, as stated in Section~\ref{sec:setup}. For the two-axis certificate, $H_f$ and $H_r$ denote the
smaller of the drive and braking angular-momentum capacities on the front and
rear axles, respectively.

\begin{theorem}[Authority ceiling and impossibility bound]
\label{thm:budget}
Under A1--A4, the maximum recoverable change of body pitch-rate over any
flight is direction-dependent:
\begin{align}
B_\uparrow&=\frac{\sum_i I_i(\omega_{\max}-\omega_{i,0})}{I_{yy}}
\quad\text{(nose-up, drive)},\\
B_\downarrow&=\frac{\sum_i I_i(\omega_{i,0}-\omega_{\min})}{I_{yy}}
\quad\text{(nose-down, braking)}
\end{align}
where $\omega_{i,0}$ is the takeoff spin speed of wheel~$i$ and
$\omega_{\min}$ the lower wheel spin bound.
Independent of motor torque, any takeoff whose required correction exceeds
the budget in its direction by more than the terminal-window half-width
$\bar\omega_\theta$,
$\omega_{\mathrm{target}}-\omega_{\theta0}>B_\uparrow+\bar\omega_\theta$ or
$\omega_{\theta0}-\omega_{\mathrm{target}}>B_\downarrow+\bar\omega_\theta$,
cannot be brought into the terminal-rate window
$|\omega_\theta(T)-\omega_{\mathrm{target}}|\le\bar\omega_\theta$ in flight.
Let $B\in\{B_\uparrow,B_\downarrow\}$ denote the budget in the required
correction direction; then $B\propto 1/I_{yy}$ and, for fixed wheel hardware
with fixed inertias and speed bounds, and body inertia growing with mass at fixed
footprint ($I_{yy}\propto m$), $B$ decreases with vehicle mass.
\end{theorem}
\begin{proof}
Projecting A2 onto the pitch axis, with wheel spin positive in the
forward-rolling sense and body pitch rate nose-up positive, gives
$I_{yy}\Delta\omega_{\theta}=\sum_iI_i\Delta\omega_i$ at $\delta=0$; both
sign conventions flip relative to the body $+y$ axis, so the reaction enters
with a positive sign, and driving the wheels forward pitches the nose up.
Applying the A4 wheel-speed bounds yields the two stated extrema; for
steered flights the front spin axes enter the pitch projection with factor
$\cos\delta\le1$ and divert part of the exchange to the roll axis, so
every admissible steering profile is dominated by the $\delta\equiv0$
extremum. Motor
torque controls only how quickly this finite momentum is exchanged; if even
the extremal terminal rate misses the allowed window, no admissible policy
can reach it. Finally, fixed wheel inertias and $I_{yy}\propto m$ give
$B\propto1/m$.
\end{proof}

This formalizes the claim of Section~\ref{sec:intro} that the decisive
leverage lies before takeoff: it is provably impossible to recover a large
takeoff disturbance in flight, regardless of actuator strength, and the
unrecoverable regime grows with scale. Since $\omega_{\min}<0$, the braking
budget $B_\downarrow$ already includes regeneratively braking the wheels
through zero and driving them in reverse. Each budget is a wheel
angular-momentum increment (kg\,m$^2$\,rad/s) divided by the body inertia
$I_{yy}$ and thus carries units of \emph{body} angular rate, i.e., the body
pitch-rate change that the wheels can absorb or supply. With
$\omega_{\mathrm{target}}=0$ the terminal-rate window coincides with the
$\CL$ rate window.
The evaluation checks torque invariance, over-budget residuals, and
direction dependence.

\begin{corollary}[Speed--authority coupling]
\label{cor:asym}
Under A5, the directional budgets $B_\uparrow$ and $B_\downarrow$ of
Theorem~\ref{thm:budget} are set by the takeoff ground speed. Since the
takeoff wheel speed satisfies
$\omega_{i,0}=v_0/r$, where $r$ is the wheel radius, the pre-takeoff speed
selects not only the ballistic arc but also the split of in-flight authority:
\begin{equation}
\frac{dB_\uparrow}{dv_0}=-\frac{\sum_i I_i}{r\,I_{yy}}
=-\frac{dB_\downarrow}{dv_0}
\end{equation}
where $B_\uparrow$ and $B_\downarrow$ are the nose-up and nose-down
pitch-rate budgets of Theorem~\ref{thm:budget}.
\end{corollary}
\begin{proof}
Immediate from the wheel-speed bounds in the proof of
Theorem~\ref{thm:budget}, applied separately per direction, with
$\omega_{i,0}=v_0/r$ at lip departure because the wheels roll without slip on
the run-up.
\end{proof}

Braking toward $v_{\mathrm{crit}}$ during the run-up therefore serves two
purposes: it shapes the ballistic arc and restores the nose-up budget
$B_\uparrow$, the direction required to arrest a nose-down takeoff rate. This
is the speed--authority coupling exploited by the gate in
Section~\ref{sec:shaping}.

In flight, the steered front wheels resolve a reaction of bounded magnitude
into a pitch component proportional to $\cos\delta$ and a roll component
proportional to $\sin\delta$,
so the two axes draw on a shared front-axle budget.

\begin{theorem}[Conservative coupled two-axis control-authority frontier]
\label{thm:frontier}
Given the direction-symmetric axle capacities $H_f$ and $H_r$, let
$T_H=\max\{H_f/(n_f\tau_{\max}),\,H_r/(n_r\tau_{\max})\}$ denote the
time for the more heavily loaded axle to spend its full momentum capacity at
per-wheel torque $\tau_{\max}$, the two axles spending in parallel. Then, for
any flight of duration $T\ge T_H$,
the following set of pitch- and roll-rate changes is certified reachable
within the flight:
\begin{multline}
(\Delta\omega_{\mathrm{pitch}},\Delta\omega_{\mathrm{roll}})\in
\Big\{\Big(\tfrac{h\cos\delta+p_r}{I_{yy}},\,
\tfrac{h\sin\delta}{I_{xx}}\Big):\\
|h|\le H_f,\,|\delta|\le\delta_{\mathrm{roll}}^{\max},\,|p_r|\le H_r\Big\}
\end{multline}
where $I_{xx}$ is the body roll inertia, $H_f$ and $H_r$ are the front- and
rear-axle angular-momentum budgets, $h$ the front-axle momentum expenditure,
$\delta_{\mathrm{roll}}^{\max}$ the
steering limit, $p_r$ the rear-axle pitch-only contribution, and
$(\Delta\omega_{\mathrm{pitch}},\Delta\omega_{\mathrm{roll}})$ the recoverable
pitch- and roll-rate changes. For each fixed $p_r$ the attainable
boundary ($|h|=H_f$) is the arc of the ellipse
$(I_{yy}\,\Delta\omega_{\mathrm{pitch}}-p_r)^2+(I_{xx}\,\Delta\omega_{\mathrm{roll}})^2=H_f^2$
within the steering sector $|\delta|\le\delta_{\mathrm{roll}}^{\max}$.
\end{theorem}
\begin{proof}
Because each axle spends its own capacity in parallel at per-wheel torque
$\tau_{\max}$, the condition $T\ge T_H$ permits every admissible capacity
expenditure. Steering
maps the front momentum $h$ through $(\cos\delta,\sin\delta)$, filling a
sector in momentum space; scaling by $(1/I_{yy},1/I_{xx})$ gives the stated
elliptical sector restricted to $|\delta|\le\delta_{\mathrm{roll}}^{\max}$,
and the rear axle adds the pitch segment
$p_r\in[-H_r,H_r]$. Since $|\cos\delta|<1$ for a nonzero roll allocation,
pitch and roll share rather than duplicate the front budget.
\end{proof}

The certified set is a coupled frontier, not a product box. Allocating
authority to roll reduces available pitch authority because
$\cos\delta<1$ whenever $\sin\delta\neq0$. At the nominal parameters,
$T_H\approx0.1$\,s, below any
jump flight time.
This frontier makes the allocation question explicit. Pitch-only and roll-only
each forfeit the other degree of freedom, whereas dual-axis control can allocate
the shared budget when both axes bind. When roll does not bind, pitch-only
allocation retains the full pitch budget.

\begin{corollary}[Feasible-set contraction under sustained roll demand]
\label{cor:camber}
If roll regulation occupies a sustained front-steer allocation
$\bar\delta\neq0$ over the flight, for example when nulling roll injected by
a banked run-up, the pitch budget attainable with the front axle held at
$\bar\delta$ is bounded above by
\begin{equation}
B_{\mathrm{eff}}(\bar\delta)=\frac{H_r+H_f\cos\bar\delta}{I_{yy}}
\;<\;B_{\mathrm{sym}}
\label{eq:beff}
\end{equation}
where $\bar\delta$ is the sustained roll-steer allocation,
$H_f$ and $H_r$ are the front- and rear-axle momentum budgets from
Theorem~\ref{thm:frontier},
$B_{\mathrm{sym}}=(H_r+H_f)/I_{yy}$ the zero-steer symmetric budget, and
$B_{\mathrm{eff}}$ the contracted pitch budget. The latter decreases
monotonically in $|\bar\delta|$. Substituting
$B_\uparrow=B_\downarrow=B_{\mathrm{eff}}$ into the certificate of
Theorem~\ref{thm:feasible}, stated below, contracts the certified set
accordingly: $\Fset(\bar\delta)\subseteq\Fset$, strictly in the rate
condition and, wherever the displacement cap of Eq.~\eqref{eq:dispcap}
binds, in the displacement condition as well.
\end{corollary}
\begin{proof}
\looseness=-1
Theorem~\ref{thm:frontier} scales front pitch authority by
$\cos\bar\delta$; monotonicity of the rate margin and $D(T_c,B_r)$ gives
the nested certified set.
\end{proof}

Equation~\eqref{eq:beff} is a frontier-restricted upper envelope rather
than an independently spendable pitch budget: when the roll loop also
prescribes the front-axle expenditure, the front-axle pitch contribution
becomes a by-product whose sign is fixed by the roll demand, and only the
rear-axle term remains freely assignable. We therefore use
Eq.~\eqref{eq:beff} to predict contraction trends; the joint certification
condition appears in Corollary~\ref{cor:rollaxis}.

Since the run-up cross-slope $\gamma$ sets the sustained roll demand,
$\Fset$ contracts monotonically with $\gamma$. This is the analytic
prediction of the feasible-set contraction observed under increasing run-up
cross-slope.

\begin{theorem}[Constructive closed-form feasibility certificate]
\label{thm:feasible}
Under A1--A4, for the budget-constrained double integrator with
$a_\theta=n_w\tau_{\max}/I_{yy}$ and flight time $T$, define
\begin{align}
e_0&=\omega_{\theta0}-\omega_{\mathrm{target}},\\
T_r&=|e_0|/a_\theta,\qquad T_c=T-T_r,\\
r_\theta&=\theta_{\mathrm{surf}}-\theta_0-\omega_{\mathrm{target}}T
          -\frac{e_0|e_0|}{2a_\theta}.
\end{align}
After rate nulling, the direction-aware residual rate margin is
\begin{equation}
B_r(e_0,r_\theta)=
\begin{cases}
B_\uparrow+e_0, & r_\theta\ge0,\\
B_\downarrow-e_0, & r_\theta<0.
\end{cases}
\label{eq:residualbudget}
\end{equation}
The null-then-correct policy certifies the closed-form inner set
\begin{multline}
\Fset_\theta=\Big\{(\theta_0,\omega_{\theta0}):
-B_\uparrow\le e_0\le B_\downarrow,\quad T_r\le T,\ \text{and}\\
|r_\theta|\le D\!\left(T_c,B_r(e_0,r_\theta)\right)\Big\}
\subseteq\Pi_\theta(\Fset^\star)
\label{eq:certifiedset}
\end{multline}
where
\begin{equation}
D(T_c,B_r)=
\begin{cases}
\tfrac14 a_\theta T_c^2, & T_c<2B_r/a_\theta,\\[3pt]
B_rT_c-\dfrac{B_r^2}{a_\theta}, & T_c\ge 2B_r/a_\theta.
\end{cases}
\label{eq:dispcap}
\end{equation}
The construction reaches
$(\theta_{\mathrm{surf}},\omega_{\mathrm{target}})$ exactly, so it satisfies any
terminal window containing that pair. Here $\Pi_\theta$ denotes projection
onto $(\theta_0,\omega_{\theta0})$. Eq.~\eqref{eq:certifiedset} holds for
takeoff states whose remaining coordinates admit an admissible completion: the roll pair $(0,0)$ with $\delta\equiv0$ and a takeoff
speed consistent with $T$ whose touchdown vertical rate meets the
$\bar v_z$ window of $\CL$.
\end{theorem}
\begin{proof}
Set $q=\omega_\theta-\omega_{\mathrm{target}}$, so
$\dot q=u$. Theorem~\ref{thm:budget} permits nulling $q(0)=e_0$ exactly
under the stated directional bounds. Applying
$u=-a_\theta\operatorname{sgn}(e_0)$ for $T_r$ contributes
\begin{equation}
\int_0^{T_r}q(t)\,dt=\frac{e_0|e_0|}{2a_\theta}
\label{eq:nullphasearea}
\end{equation}
and leaves $r_\theta$ for the remaining $T_c$. Nulling $e_0$ spends momentum
on one side of the wheel-speed envelope and releases an equal increment on
the other, so the margin available for the correction phase is the
direction-aware quantity of Eq.~\eqref{eq:residualbudget}. A symmetric
accelerate--coast--decelerate profile has peak
$\min(B_r,a_\theta T_c/2)$; integrating its triangular or trapezoidal rate
profile gives Eq.~\eqref{eq:dispcap}. Thus every state in
Eq.~\eqref{eq:certifiedset} has a constructive trajectory to the terminal
pair.
\end{proof}

The certificate is one-sided. A state outside
$\Fset_\theta$ is uncertified by this construction; only violation of the
directional rate condition $-B_\uparrow\le e_0\le B_\downarrow$ proves that
no admissible policy reaches the exact terminal rate
$\omega_{\mathrm{target}}$, and when the violation exceeds
$\bar\omega_\theta$, the terminal-rate window itself is unreachable by
Theorem~\ref{thm:budget}. When the certified speed set is empty,
$\Fset^\star$ is empty and the inclusion is vacuous.

Instantiated with the nominal platform
parameters of Section~\ref{sec:setup}, a low initial rate, and $T=1.5$\,s,
the displacement cap evaluates to $D\approx18^\circ$, the takeoff pitch
error the construction can absorb in flight. A numerical cross-validation
compares the certificate against the exact reachable set, obtained by a
400-step linear-programming (LP) computation in
18 parameter configurations. No certified state falls outside the exact set, and the
certificate covers 89--96\% of its displacement interval, with a median of
94\%, so
the closed form is sound and loses little to conservatism.

\begin{corollary}[Pitch-correction upper envelope]
\label{cor:envelope}
Under a fixed-direction budget $B$ and angular-acceleration bound $a_\theta$, and without a
terminal-rate constraint, the maximum pitch correction achievable over a
flight of duration $T$ is
\begin{equation}
\Delta\theta_{\max}(T)=
\begin{cases}
\tfrac12 a_\theta T^2, & T<B/a_\theta,\\[3pt]
BT-\dfrac{B^2}{2a_\theta}, & T\ge B/a_\theta
\end{cases}
\label{eq:envelope}
\end{equation}
where $\Delta\theta_{\max}(T)$ is the maximum achievable pitch correction
over flight duration $T$.
\end{corollary}
\begin{proof}
A bang-coast profile maximizes terminal displacement. The vehicle
accelerates at $a_\theta$ until the rate change reaches the budget $B$ at
$t=B/a_\theta$ and then coasts at that rate; integrating the acceleration and
coast segments yields the two stated branches.
\end{proof}

This envelope is a relaxed upper bound on the accelerate--coast--decelerate
displacement in Eq.~\eqref{eq:dispcap}. The correction is torque-limited only
on flights shorter than $B/a_\theta$ and budget-limited otherwise.
The same nominal instantiation locates the operating branch. With the
platform parameters of Section~\ref{sec:setup},
$a_\theta\approx135^\circ$/s$^2$ and $B_\uparrow\approx13^\circ$/s, so
$B_\uparrow/a_\theta\approx0.10$\,s. A $T=1$\,s flight lies on the
budget-limited branch, and the envelope caps the correction at
$\Delta\theta_{\max}\approx12.4^\circ$.

\begin{corollary}[Per-axis roll-attitude certificate]
\label{cor:rollaxis}
Under A1--A4 the roll axis is the same budget-constrained double integrator
with acceleration bound $a_\phi=n_f\tau_{\max}\sin\delta_{\mathrm{roll}}^{\max}/I_{xx}$
and rate budget $B_\phi=H_f\sin\delta_{\mathrm{roll}}^{\max}/I_{xx}$, which is
the front-axle momentum budget resolved through the steering limit per
Theorem~\ref{thm:frontier}; here $n_f$ is the number of front wheels. The
roll pair of the terminal set, $|\phi|\le\bar\phi$ and
$|\omega_\phi|\le\bar\omega_\phi$, is therefore
certified by the closed form of Theorem~\ref{thm:feasible} with
$(\theta,\theta_{\mathrm{surf}},\omega_{\mathrm{target}})$ replaced by
$(\phi,0,0)$, the acceleration bound by $a_\phi$, both directional budgets
set to $B_\phi$, and the certified set enlarged by the terminal windows
$(\bar\phi,\bar\omega_\phi)$. A sufficient joint condition follows from
Theorem~\ref{thm:frontier}: the net momentum pair demanded by the two
per-axis constructions,
$(I_{yy}\Delta\omega_{\mathrm{pitch}}-p_r,\;I_{xx}\Delta\omega_{\mathrm{roll}})$
for some $|p_r|\le H_r$, must lie in the frontier sector
$\{h(\cos\delta,\sin\delta):|h|\le H_f,\,|\delta|\le\delta_{\mathrm{roll}}^{\max}\}$,
checked separately for the null and correction phases.
\end{corollary}
\begin{proof}
Resolving the front budget through $\sin\delta$ gives the same bounded
double integrator as Theorem~\ref{thm:feasible}; variable substitution
proves the per-axis certificate, while Theorem~\ref{thm:frontier} supplies
the shared-budget coupling. Sequential scheduling of the two axes within $T$ is feasible in time
because each axle spends its capacity within the torque-limited time $T_H$
of Theorem~\ref{thm:frontier}, far below jump flight times; joint
feasibility additionally requires the frontier-membership condition stated
above, which accounts for the pitch by-product of the front-axle roll
expenditure.
\end{proof}

This corollary closes the formal gap so that every attitude component of
$\CL$ carries a certificate. The joint pitch--roll certified feasible set,
however, is not the product of the two per-axis sets: both axes draw on the
shared front-axle budget, so the product is only an outer approximation.
Substituting the contracted budget $B_{\mathrm{eff}}(\bar\delta)$ of
Corollary~\ref{cor:camber} into the pitch condition predicts the
contraction trend under a sustained roll allocation, but it is not by itself
a joint certificate, because a roll maneuver with net front-axle expenditure
$h$ also injects the pitch by-product $h\cos\delta$.

\begin{proposition}[Ballistic envelope and gate selection]
\label{prop:env}
Assume the landing surface is a descending terrain profile that each
parabola crosses transversally at a unique touchdown point, so the touchdown
point varies continuously with the takeoff speed. The forward ballistic
image of $\Fset$, parameterized by takeoff speed
$v\in[v_{\mathrm{lo}},v_{\mathrm{hi}}]$ and launch angle, is a family of center-of-mass
parabolas. For a fixed speed, the family over launch angles is bounded above by the
classical safety parabola
$z=z_{\mathrm{lip}}+v^2/(2g)-g(x-x_{\mathrm{lip}})^2/(2v^2)$, where $g$
is gravitational acceleration, $z_{\mathrm{lip}}$ and $x_{\mathrm{lip}}$
the lip height and along-track position, and
$[v_{\mathrm{lo}},v_{\mathrm{hi}}]$ the evaluated takeoff-speed range;
with the launch angle fixed by the ramp,
the family is indexed by $v$ alone and varying $v$ sweeps the certified
landing region. Along
each parabola the touchdown vertical rate is a continuous function of the
takeoff speed,
\begin{align}
v_z(v)&=v\sin\alpha_{\mathrm{eff}}-g\,T(v),
\label{eq:vzofv}\\
T(v)&=\frac{v\sin\alpha_{\mathrm{eff}}
+\sqrt{(v\sin\alpha_{\mathrm{eff}})^2+2g\,\Delta z(v)}}{g}\notag
\end{align}
where $\alpha_{\mathrm{eff}}$ is the effective launch angle at the lip,
which is set by the ramp and distinct from the run-up cross-slope $\gamma$,
$T(v)$ the flight time, and $\Delta z(v)\ge0$ the height drop from the lip
to the touchdown point. This drop is determined by the intersection of the
parabola with the landing surface and thus incorporates the landing geometry.
The
$\CL$ condition $|v_z|\le\bar v_z$ restricts $v$ to a closed subset of the
evaluated range. A sufficient
certificate for this speed component of $\Fset$ is obtained when the landing
surface intersects this
region at a point admissible for $\CL$.
\end{proposition}
\begin{proof}
Under the transversality assumption, the ballistic map and $v\mapsto v_z(v)$
are continuous; intersecting their
closed constraints with the compact evaluated speed range makes the
certified speed projection compact. It may be empty or disconnected, so its
largest element below any given request exists whenever the set contains
a point at or below that request.
\end{proof}

The gate operates on this certified speed set: with the launch angle fixed
by the ramp geometry, it selects the largest certified speed below the
performance request and issues no-go when no such speed exists.
Because leaving $\Fset$ removes the constructive feasibility certificate,
gate design must distinguish performance from safety margin. States that also
violate Theorem~\ref{thm:budget} cannot be repaired by additional in-flight
actuator effort. This distinction defines two subsets of $\Fset$.
\begin{proposition}[Performance-optimal versus safety-maximal envelopes]
\label{prop:safety}
The \emph{performance-optimal envelope} consists of points on
$\partial\Fset$ that maximize the selected performance objective, namely
takeoff speed, jump distance, or traverse rate. It offers zero certified margin.
The \emph{safety-maximal envelope} is the eroded set
$\Fset_{\mathrm{safe}}=\Fset\ominus \mathcal{B}_\rho$, where
$\mathcal{B}_\rho$ is a takeoff-state uncertainty ball of radius $\rho$.
Its safety-optimal points maximize distance from the boundary,
$\arg\max_{x\in\Fset}\mathrm{dist}(x,\partial\Fset)$; uniqueness is not
assumed.
\end{proposition}
\begin{proof}
Compactness of a nonempty $\Fset$ makes the distance-to-boundary function
attain its maximum; standard erosion properties give the monotone shrinkage
of $\Fset_{\mathrm{safe}}$ with $\rho$.
\end{proof}
\begin{figure}[!t]
\centering
\includefig[width=\columnwidth]{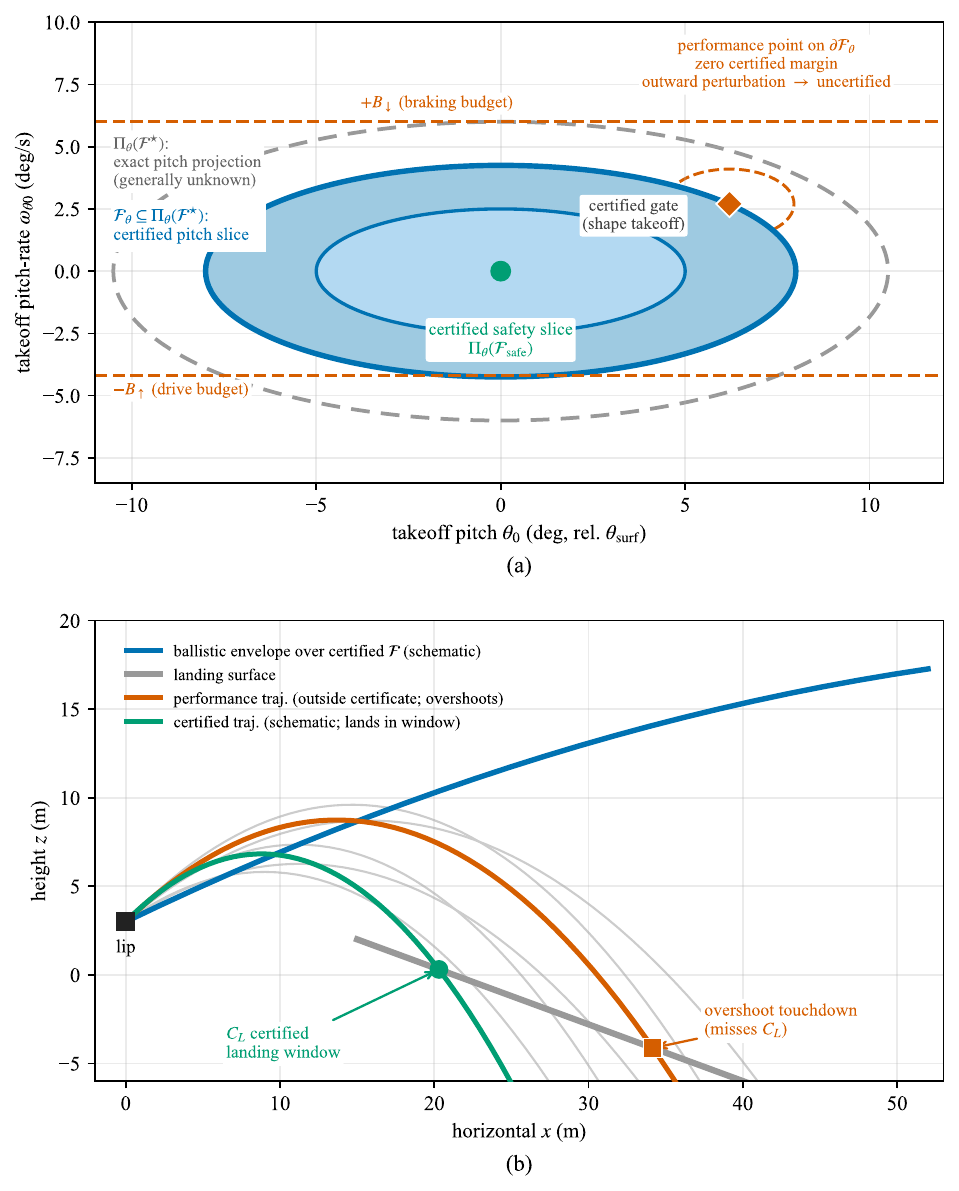}
\caption{Schematic certified takeoff envelopes and ballistic
back-propagation (not to scale). (a)~Pitch-state sets and gate selection.
(b)~Certified and overshooting ballistic trajectories.}
\label{fig:envelopes}
\end{figure}

For $\omega_{\mathrm{target}}=0$, Fig.~\ref{fig:envelopes}(a) overlays the
exact projection $\Pi_\theta(\Fset^\star)$, closed-form slice $\Fset_\theta$,
directional bounds $[-B_\uparrow,B_\downarrow]$, and robust interior
$\Pi_\theta(\Fset_{\mathrm{safe}})$. Panel (b) maps the certified set forward:
the boundary trajectory overshoots, whereas an interior trajectory reaches
$\CL$.
Performance pulls toward $\partial\Fset$ while robustness pulls inward: at a
performance-optimal point, an outward takeoff-state perturbation may exit
$\Fset$, at which point the landing certificate no longer applies. The
evaluated regulator selects the largest certified speed; erosion by a calibrated
$\mathcal B_\rho$ is its uncertainty-aware extension. If the eroded set is
empty, the vehicle must shape earlier or abort. Figure~\ref{fig:envelopes}
schematizes Propositions~\ref{prop:env} and~\ref{prop:safety}.

\begin{proposition}[Bounded saturated rate tracking]
\label{prop:ratebounded}
Let $\kappa=n_w K_d/I_{yy}$ denote the inner-loop rate-tracking bandwidth.
Under A1, the timescale separation $\kappa\gg K_\theta$, and within the
momentum budget of Theorem~\ref{thm:budget}, the saturated
rate-tracking law of Section~\ref{sec:dual} contracts its rate error outside
the standard cascade boundary layer and remains bounded over the finite flight
horizon.
\end{proposition}
\begin{proof}
Take $V=\tfrac12 e_\omega^2$ with
$e_\omega=\omega_\theta^{\mathrm{des}}-\omega_\theta$. While the rate
reference is saturated, $\omega_\theta^{\mathrm{des}}$ is constant, so
$\dot e_\omega=-\kappa e_\omega$ and $\dot V=-\kappa e_\omega^2$; where the
torque bound binds instead, $\dot\omega_\theta=a_\theta\operatorname{sgn}(e_\omega)$
and $\dot V=-a_\theta|e_\omega|\le0$. In the unsaturated region the extra term
$-K_\theta\omega_\theta e_\omega$ is dominated outside a boundary layer of
size $O((K_\theta/\kappa)|\omega_\theta|)$. Finite acceleration and rate
bounds then imply finite-horizon boundedness.
\end{proof}

This closed-loop property does not enlarge the certified pitch set
$\Fset_\theta$ of Theorem~\ref{thm:feasible}.

\begin{corollary}[Cross-phase back-propagation]
\label{cor:backprop}
Projecting $\Fset$ onto the speed axis for the current geometry gives a compact
certified speed set. When this set contains a point at or below the
performance request, its largest such element $v_{\mathrm{crit}}$ exists, and
propagating this boundary through the ground run-up
yields the admissible speed envelope of Eq.~\eqref{eq:vmaxint}.
\end{corollary}
\begin{proof}
The projection defines $v_{\mathrm{crit}}$. With $a_{\mathrm{brake}}(s)>0$
the total decelerating specific force of Section~\ref{sec:shaping}, the
work--energy relation over the remaining distance $d$ gives
Eq.~\eqref{eq:vmaxint}: below the envelope, the vehicle can shed the excess
speed to $v_{\mathrm{crit}}$ within the remaining run-up.
\end{proof}

When the certified speed set contains no point at or below the request, the
gate issues no-go. This corollary is the constructive response to the
impossibility result of Theorem~\ref{thm:budget}: the takeoff state is shaped
into $\Fset$ before launch, where in-flight correction is certified.

\section{Experimental Setup}
\label{sec:setup}

We run all vehicle-dynamics experiments in BeamNG.tech version 0.38.3.0~\cite{beamngtech},
a soft-body simulator previously used in off-road navigation
studies~\cite{zhao2026jump}. Our simulated full-scale platform is
a custom 4WIDS electric vehicle whose curb mass is \mbox{$\approx$1383\,kg}, read back via
BeamNG \texttt{calcBeamStats}; the approach-mode mass ladder of
Table~\ref{tab:scenarios} spans $\pm20\%$ about this value.
Figure~\ref{fig:vehframe} shows the vehicle, rendered in the red livery used
for DART in the multi-car runs, together with its body frame and the measured
inertial and drivetrain parameters used throughout the analysis. The body
frame is $x$ forward, $y$ left, and $z$ up; nose-up pitch $\theta$ about
$-y$ is positive;
the overlay is computed from BeamNG node masses and drivetrain limits.

We evaluate in two complementary modes. Air-impulse experiments initialize
the vehicle at the lip with prescribed takeoff states, isolating the
airborne phase (P1). Approach experiments instead generate the takeoff
state through the deployment-relevant ground run-up; one of them evaluates
the complete cross-phase chain, including reachability gating, airborne
control, and landing, end to end on a single vehicle and geometry.
Direction-resolved budget tests separately check the analytic
feasible-set criterion. Our primary metrics are absolute landing pitch and
roll error, safe-landing rate, and touchdown speed, defined as the total
center-of-mass speed at first ground contact. Touchdown speed is dominated
by the horizontal component; in the steep-lip scenario, for example, the
median touchdown speed is ${\approx}19$\,m/s while the median vertical
component is ${\approx}12$\,m/s. We report three nested landing outcomes.
An \emph{on-target} landing reaches the simulator's instrumented landing
surface without tumbling. A \emph{safe} landing is on-target with
$|$pitch error$|<35^\circ$ and $|$roll$|<30^\circ$; these are the
crash-avoidance bounds used throughout Section~\ref{sec:results}, within
which the suspension can absorb the impact. A \emph{clean} landing is
on-target with $|$pitch error$|\le5^\circ$. Confirmatory comparisons use
$N\ge30$ trials per configuration and report medians; the figures show
dispersion. Each table states whether the protocol is same-tick paired or
single-vehicle interleaved and gives the corresponding significance test.

Table~\ref{tab:scenarios} reports launch mode before the semicolon in its
Mode/protocol column and comparison protocol after it. Air-impulse injects
the takeoff state at the lip, whereas approach generates it through the
run-up. Interleaved tests rotate laws on one vehicle and include a same-law
replicate; same-tick tests launch the compared laws simultaneously.
Confirmatory experiments target $N=30$ per setting, geometric generalization
yields 27--30 valid paired trials, and ``r'' marks same-state physics repeats.
The resonance entry denotes the $\sigma=0.25$, 20-ms condition of
Section~\ref{sec:res-jitterlat}.

\begin{figure}[!t]
\centering
\includefig[width=\columnwidth]{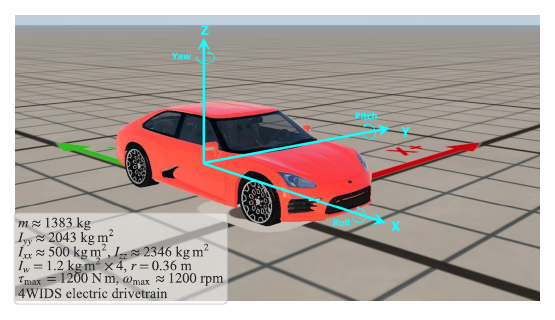}
\caption{Experimental 4WIDS platform, body/inertial frames, and measured
parameters used by the authority analysis.}
\label{fig:vehframe}
\end{figure}

\begin{table*}[!t]
\centering
\caption{Experimental design matrix.}
\label{tab:scenarios}
\scriptsize
\setlength{\tabcolsep}{2pt}
\begin{tabular}{@{}>{\raggedright\arraybackslash}p{0.16\textwidth}>{\raggedright\arraybackslash}p{0.21\textwidth}>{\raggedright\arraybackslash}p{0.30\textwidth}>{\raggedright\arraybackslash}p{0.18\textwidth}>{\raggedright\arraybackslash}p{0.11\textwidth}@{}}
\toprule
Experiment & Mode / protocol & Geometry / disturbance & Tests & $N$ \\
\midrule
\multicolumn{5}{@{}l}{\textit{Tier 1: theory validation; DART only; exact-state injection}} \\
Pitch-authority envelope & air-impulse; DART only & $\theta_0{=}+10^\circ$--$+45^\circ$, 7 levels & Thm.~\ref{thm:budget} & 5r/level \\
Directional budget probes & air-impulse; DART only & nose-up task at $v_0{=}16/20$\,m/s; nose-down ladder & Thm.~\ref{thm:budget}, Cor.~\ref{cor:asym} & 12r/speed + ladder \\
Certified-boundary grid & air-impulse, torque-free; DART only & 25 cells: 15 boundary + 10 interior; $\theta_0\times e_0$ & Thm.~\ref{thm:feasible} & 5r/cell; 125 total \\
Gap-clearance calibration & air-impulse; same-state ladder & launch $\{8,12\}^\circ$, gaps $\{33,54\}$\,m, seven speeds & Prop.~\ref{prop:env} & 2r/speed \\
No-go demonstration & approach; single vehicle & $\alpha{=}20^\circ$, empty 54-m window, $v_{\mathrm{req}}{=}22$\,m/s & Cor.~\ref{cor:backprop} & 3 \\
\midrule
\multicolumn{5}{@{}l}{\textit{Tier 2: decisive comparisons; DART vs.\ TOBB / RW-PD}} \\
Air-impulse head-to-head & air-impulse, Cond.~1; interleaved & $\alpha{\in}\{20,24,28\}^\circ$ & Regime B & 30/law/$\alpha$ \\
Approach head-to-head & approach; same-tick & $20^\circ$ tabletop steep lip, Regime A; $24^\circ$ curved lip, Regime D & Regimes A / D & 30/law/cell \\
Banked run-up sweep & approach; same-tick, ramp copies & $\alpha{=}20^\circ$, $\gamma{\in}\{4,8,12\}^\circ$, 50\,m run-up & Cor.~\ref{cor:camber}, Regime C & 30/law \\
End-to-end gate & approach; single vehicle & $\alpha{=}20^\circ$ steep lip, gate off/on, $v_{\mathrm{req}}\sim\mathcal{U}(21,23)$ & Cor.~\ref{cor:backprop} & 30/config \\
Banked gate-off replicate & approach; same-tick, ramp copies & $\gamma{=}8^\circ$, reachability gate off & gate / airborne attribution & 30/law \\
Per-axis ablation & air-impulse, Cond.~1; interleaved & $\alpha{=}24^\circ$, pitch-only / dual / roll-only & allocation test with null result & 30/variant \\
Per-flight latch ablation & mixed; same-tick & flat approach / $28^\circ$ air / resonance / $1.2\times$ mass & Thm.~\ref{thm:frontier}, Eq.~\eqref{eq:wroll} & 30/config \\
\midrule
\multicolumn{5}{@{}l}{\textit{Tier 3: generalization axes; one axis varied; reference configuration otherwise}} \\
Geometric generalization & 4 approach + 3 air-impulse; same-tick & angle / rise axes; see Table~\ref{tab:geomgen} & severity / airtime laws & 27--30/law/geom. \\
Gain sweep & air-impulse, Cond.~1; same-tick & $K_\theta,K_\phi,K_d$ at $0.5\times$--$2\times$ & DART gain robustness & 30/setting \\
Platform sweep & air-impulse, Cond.~1; same-tick & torque, tire pressure, CoM, $I_{yy}$ & platform robustness & 30/setting \\
Run-up surface & approach; same-tick & asphalt / gravel / dirt & traction robustness & 30/law/surface \\
Mass ladder & approach; same-tick & $0.8\times/1.0\times/1.2\times$ scaled $m/k/c$ & budget $\propto$ inertia ratio & 30/law/mass \\
Takeoff-state jitter & air-impulse, Cond.~1; same-tick & $\sigma{\in}\{0.25,1,2,3,5\}$; latency fixed at 20\,ms & pitch robustness & 30/law/$\sigma$ \\
Actuator latency & air-impulse, Cond.~1; interleaved & latency $\{0,10,20,40\}$\,ms; $\sigma{=}0.25$ fixed & yaw-resonance exposure & 30/law/lat. \\
\bottomrule
\end{tabular}
\end{table*}

\subsection{Platform and Airborne Protocol}
Figure~\ref{fig:scenario-library} visualizes the geometry axes summarized in
Table~\ref{tab:scenarios}. The jump library combines
curved-lip \emph{kicker} ramps (angle $\alpha$ and
rise) with valley, gap, or flat landing surfaces. \emph{Air-impulse} mode
prescribes $(\theta_0,\omega_{\theta0},v_0,\phi_0)$ at the lip to isolate the
airborne law; \emph{approach} mode generates the takeoff state through the
run-up, climb, and release to evaluate the deployed cross-phase chain. The mass ladder
uses approach mode and consistent $m/k/c$ scaling so takeoff speed, modal
frequencies, and damping remain comparable across masses.
In Fig.~\ref{fig:scenario-library}, row 1 contains the $\alpha=20^\circ$
anchor, the $\alpha=24^\circ/28^\circ$ variants, and the tabletop steep-lip
construction. Row 2 contains the $\gamma=8^\circ$ banked run-up, the
$\beta=-24^\circ$ landing, and the $\alpha=10^\circ/14^\circ$ variants;
row 3 contains $\alpha=22^\circ$, rise 4/13\,m at $\alpha=14^\circ$, and
the gap landing.

All controllers share the same per-wheel torque/steer interface. DART
denotes the cross-phase framework with its per-flight roll latch enabled;
always-on ablations and end-to-end evaluations are named explicitly. The
latch uses $\phi_{\mathrm{on}}=8^\circ$ and
$\phi_{\mathrm{db}}=2^\circ$. At 100\,Hz, \proc{TakeoffConfirmed}
requires $n_{\mathrm{ph}}=3$ no-contact samples over 30\,ms, and
\proc{EngageOK} requires $n^{\ast}=6$ samples over 60\,ms with wheel
speed below the spin-up guard $\bar\omega_w=200$\,rad/s, just under the
drivetrain hard limit; the in-flight wheel-speed bounds are
$[-1200,1200]$\,rpm. The land-match height is
$z_{\mathrm{lm}}=2.5$\,m in the reference configuration and $1.5$\,m in the
gain and platform ladders, recorded per experiment in the released
configurations; wheel-speed re-ordering below $z_{\mathrm{lm}}$ applies
mild pitch-rate damping. The heading trim uses a $2^\circ$ yaw deadband
with gain $0.55$ and steer cap $0.45$, and attenuates roll steer beyond
$8^\circ$ yaw error. At first contact the controller stops pitch
differential actuation and holds roll steer, as
Algorithm~\ref{alg:dart} specifies.

We compare against time-optimal bang-bang (TOBB) and an error-driven
reaction-wheel-style PD law (RW-PD), following AGRO's aerial PD concept
\cite{gonzalez2020agro} but implemented through the throttle/brake
interface. We tune both baselines on a separate $\alpha=18^\circ$
air-impulse band with $N=30$ per setting, spanning TOBB acceleration over
$\{60,120,240,480\}^\circ$/s$^2$ and RW-PD proportional gain over
$\{0.3,0.6,1.2,2.4\}$.

We measure pitch error from the terrain-matched target. Safe-land requires
takeoff, no tumble, and touchdown within $35^\circ$ pitch and $30^\circ$
roll; the analytic certificate instead targets the exact terminal pair of
Theorem~\ref{thm:feasible}.

The end-to-end experiment uses a calibrated one\hyb dimensional regulator that
approximates $\theta_0$ by the ramp angle, zeroes the predicted rate, and
enforces $v_{\max}(d)$. Its settings are $v_{\mathrm{crit}}=11$\,m/s,
adaptive $a_{\mathrm{brake}}=4$\,m/s$^2$, 1.5\,m coast, $T=0.95$\,s and
$a_\theta=120^\circ$/s$^2$. An empty-window run triggers the no-go branch;
online full-state gating is future work.

Our reference comparison protocol, the \emph{single-vehicle interleaved}
design, rotates laws at one physical
location and includes a bit-identical DART replicate; we judge Mann--Whitney
tests against this same-law noise floor. In conditions where every law completes
and paired trajectories remain nonchaotic, we switch to the faster
\emph{same-tick three-law launch}; banked conditions use translated copies
of the same ramp to preserve identical geometry. Stressed conditions that
violate this criterion revert to the interleaved protocol.
Table~\ref{tab:scenarios} records the choice per experiment.

Condition~1 combines an approximately $24^\circ$ takeoff with
$28$--$36^\circ$ roll toward a near-level target; Condition~2 combines a
${\approx}-24^\circ$ landing target with $12$--$36^\circ$ roll. The
$\alpha=20^\circ$ experiments span three run-up families: \emph{flat} is the
unbanked valley, \emph{steep-lip} is the sustained nose-down kicker, and
\emph{banked} adds cross-slopes of $\gamma\in\{4,8,12\}^\circ$.
The experiments use two lip constructions, distinguished here because
the lip determines the takeoff disturbance: the library ramps end in a
curved lip with a 35\,m arc that unloads the axles gradually, whereas the
steep-lip scenarios use a tabletop profile whose support ends abruptly at
the lip edge; the abrupt front-axle unloading produces their large
nose-down takeoff pitch rates. Steep-lip scenarios land on a level surface
($\beta=0$); all library-ramp scenarios share the unified
$\beta=-12^\circ$ valley landing.

\begin{figure*}[!t]
\centering
\includefig[width=\textwidth]{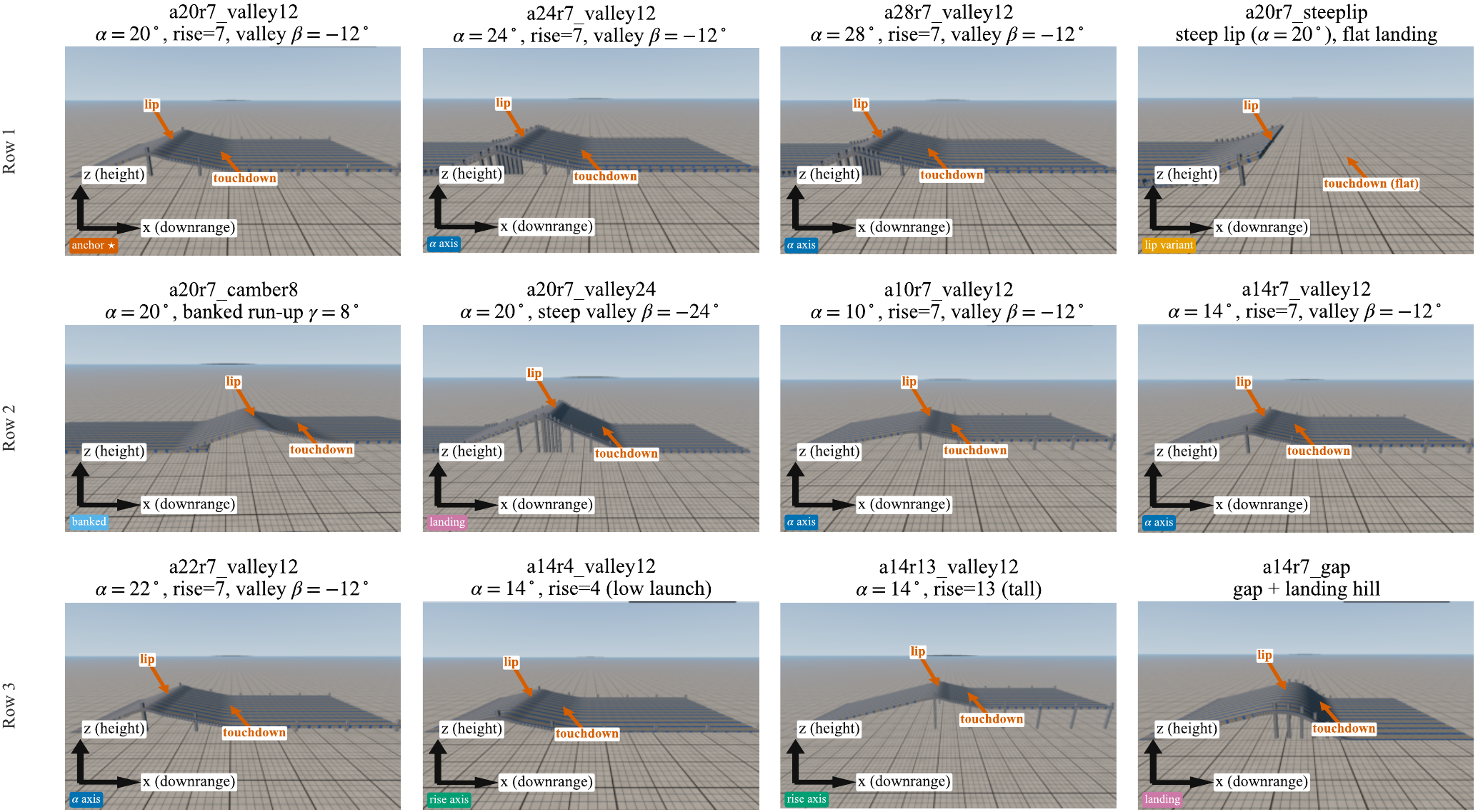}
\caption{Rendered DART jump library of twelve geometries. Orange markers denote
the takeoff lip and touchdown.}
\label{fig:scenario-library}
\end{figure*}

\subsection{Reproducibility}
We run all experiments with controlled seeds, deterministic 0.01\,s physics, and
100\,Hz control updates; the released code records every controller gain and
the seed of each experiment, so all runs can be reproduced exactly.
Confirmatory and robustness experiments use $N=30$ trials per setting, plus
a 30-jump DART replicate in interleaved protocols. The authority ladder and
certified-boundary grid instead use five same-state physics repeats per
injected state, not independent samples.

For same-tick comparisons we use two-sided Wilcoxon tests for continuous
errors and McNemar discordances for safe landing; for independent groups we
use two-sided Mann--Whitney tests. We report exact $p$-values and null
boundary results.
Source code and instructions to reproduce this work, as well as
video demos, will be available at
\url{https://github.com/MeridianCAS/DART}.

\section{Results}
\label{sec:results}

We organize the results around the claims that Table~\ref{tab:scenarios}
assigns to each experiment. The following subsections first test the mechanism
and certified envelopes, then compare DART with the baselines, vary the
generalization axes, and finally synthesize the operating regimes and
boundaries. Statistics from different launch protocols remain separate
unless stated otherwise.

\subsection{Theory Validation}
\label{sec:res-theory}
This tier tests the mechanism, angular-momentum authority, certified-state
boundary, and the empty-speed-window no-go branch before comparing
closed-loop performance.

\subsubsection{Wheel-Reaction Mechanism and Authority}
We first confirm that airborne roll authority scales as $\sin\delta$.
On a high-altitude flat platform in an open-loop test with driven wheels, sweeping
front steer \(\delta\) induces a body roll rate of \(28.2\sin\delta\)\,\(^\circ\)/s with
\(R^2=0.987\). Pure left/right torque differential at zero steer yields roll
rates below \(0.12^\circ\)/s, roughly \(110\times\) below the steer-induced
response at the largest tested steer. At \(\delta=30^\circ\), that response
is \(13.4^\circ\)/s. Airborne roll is therefore steer-induced drive reaction,
not differential thrust, consistent with the bicycle-model
\(\sin\delta\) scaling~\cite{pokhrel2025dom} at full-vehicle scale.

\begin{figure}[!t]
\centering
\includefig[width=\columnwidth]{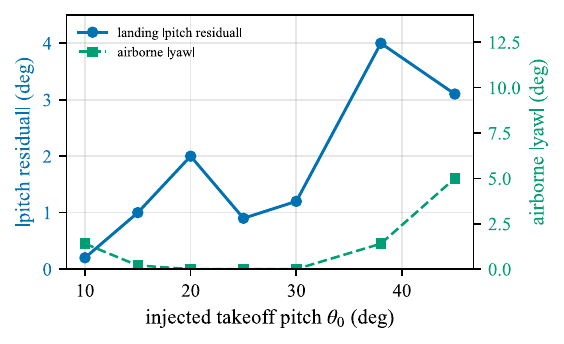}
\caption{Pitch-authority ladder for pure attitude injections from
$\theta_0=+10^\circ$ to $+45^\circ$. Each level uses five same-state repeats
with the roll latch enabled. Left: median landing-pitch residual. Right: median
airborne yaw drift measured from forward-direction azimuth.}
\label{fig:authenv}
\end{figure}

Sweeping pure pitch-attitude injections from $+10^\circ$ to $+45^\circ$
against the $-12^\circ$ landing target creates correction demands of
$22$--$57^\circ$. Each level uses $N=5$ repeats. DART with its per-flight roll latch holds the landing
residual at $0.2$--$4.0^\circ$
across the entire ladder with every jump on the landing surface, as
Fig.~\ref{fig:authenv} shows. Airborne yaw stays below $2^\circ$ through the
$+38^\circ$ level and reaches $5^\circ$ at $+45^\circ$,
still an order of magnitude below the yaw-resonance mode. With no roll
injection, the roll channel stays largely dormant and the yaw projection of
Eq.~\eqref{eq:yawleak} receives no torque activity, providing an independent
control point for the mechanism. The
large correction magnitudes are also consistent with the direction-dependent
budgets of Theorem~\ref{thm:budget} because they are nose-down corrections
executed by regenerative braking and reverse
drive, whose budget $B_\downarrow$ substantially exceeds the nose-up
$B_\uparrow$; both directions are
quantified by the budget-consistency tests reported next. This open-loop
attitude-injection ladder measures airborne authority, not gate certification or
membership in $\Fset$.

\subsubsection{Budget-Consistency Tests}
\label{sec:res-reachf1}
Two direction-resolved tests examine the angular-momentum budget where it can
be measured cleanly. In the nose-up direction under drive, the budget is
$B_\uparrow=\sum_iI_i(\omega_{\max}-\omega_{i,0})/I_{yy}$ and shrinks with
takeoff speed through $\omega_{i,0}=v_0/r$. A fixed correction task with a
$+13^\circ$ nose-up demand transitions from full correction at
$v_0=16$\,m/s to the onset of residual error at $v_0=20$\,m/s. The former
has a $1.4^\circ$ residual and $B_\uparrow\approx11^\circ$/s against an
approximately $9^\circ$/s demand; the latter has a $3.9^\circ$ residual and
$B_\uparrow\approx9.5^\circ$/s. The budget
begins to bind where the closed-form formula predicts. In the
nose-down direction, the attitude-injection ladder of the authority
envelope corrects up to $57^\circ$ over a $2.8$\,s flight, a
${\ge}20^\circ$/s average recoverable rate. This is consistent with the
reverse-inclusive braking budget of Theorem~\ref{thm:budget} at the
ladder's launch condition: the impulse launch holds drive torque through
release, so the wheels depart spinning well above rolling speed at
$\omega_{i,0}\approx165$--$190$\,rad/s, the calibrated range for this launch
mode. The resulting $B_\downarrow\approx39$--$43^\circ$/s covers the demanded
average rate with nearly a twofold margin. The certified-boundary grid below
instead uses a torque-free launch variant that leaves the wheels near rest.
We do not partition all jumps by the lip-measured
rate because that signal conflates ground-coupled transients with sustained
free-flight rotation; the steep-lip experiments, where the sustained rate is
well defined, show the budget binding as predicted.

A hard-limit calibration on the free-fall bench measures the headroom above
the sustained envelope. Holding full drive torque through flight saturates
the wheels at $221$\,rad/s, the 2105-rpm drivetrain limit, within
$0.3$\,s of command onset and keeps them there to touchdown; the
accompanying nose-up body-rate swing measures $21.7^\circ$/s from wheels at
rest and $16$--$18^\circ$/s from typical takeoff wheel speeds of
$43$--$60$\,rad/s. This is the hard-limit counterpart of the $9$--$13^\circ$/s
sustained-envelope budget of Section~\ref{sec:budget}. The steep-lip entry
rates in Section~\ref{sec:res-e2e}, whose medians span $21$--$43^\circ$/s,
exceed the
recovery budget even at the hard limit.

\subsubsection{Direct Test of the Certified-Set Boundary}
\label{sec:res-certprobe}
We test the certificate of Theorem~\ref{thm:feasible} directly by sweeping
the takeoff state over a grid that brackets the predicted boundary. The fixed
geometry uses the air-impulse protocol at $\alpha=20^\circ$ with a
$-12^\circ$ landing target. Five launch-attitude levels produce
$\theta_0=2$--$44^\circ$. Six angular-velocity impulse levels applied at
release produce sustained rates $e_0=-39$ to $+31^\circ$/s, measured by a
linear fit over the first $0.3$\,s of flight. Flight time varies with the
injected state and enters the certificate per jump. On completed flights,
$T=1.3$--$3.3$\,s. The
launch applies no drive torque, so the wheels leave the platform near rest
and
the wheel-speed terms of the budget are known measured quantities rather
than assumption A5. The 25 nominal cells comprise a 15-cell grid
bracketing the predicted boundary plus a 10-cell densification of its
certified interior; each cell is repeated five times under deterministic stepping,
giving $N{=}125$ jumps over 25 takeoff states.
Membership is evaluated per jump from its \emph{measured} takeoff state, so
replicate-to-replicate scatter can place the five jumps of a borderline cell
on both sides of the certificate boundary. A landing is scored clean when it
is on-target with $|$pitch error$|\le5^\circ$, as defined in
Section~\ref{sec:setup}.

The certificate
accepts 62 of the 125 jumps, and 61 land clean. The single certified miss
is a replicate whose flight ran $0.3$\,s longer than its cell siblings:
the correction, timed for the typical flight, carries the recovery swing
$12^\circ$ past the target by the late touchdown, still on the landing
surface and without tumbling. Both failing cells (the largest nose-up rate
injections, $e_0\approx+19$ and $+31^\circ$/s) lie outside the certificate.
Outside the certificate, 52 of 63 jumps still land clean, demonstrating that
the null-then-correct construction is conservative because the certificate demands
the exact terminal pair, whereas the outcome tolerance admits a $\pm5^\circ$
window. This degree of conservatism is consistent with the $89$--$96\%$
displacement-interval coverage of the numerical cross-validation in
Section~\ref{sec:theory}.

Two auxiliary measurements fix the certificate's parameters for this test
geometry.
A free-flight momentum-exchange probe prescribes wheel-speed swings of
${\approx}220$\,rad/s and deliberately approaches the 2105-rpm drivetrain
hard limit rather than the 1200-rpm sustained in-flight envelope of
Section~\ref{sec:budget}. The measured body-rate response yields an effective
per-corner spin inertia of $0.83$--$0.98$\,kg\,m$^2$, the wheel-side inertia
that carries the exchange in the simulator.
A wheels-at-rest glide over the same arc bounds aerodynamic and numerical
non-wheel pitch moments at ${\approx}1.5^\circ$/s$^2$, an order
of magnitude below the injected disturbances.

\subsubsection{Ballistic Speed Window and No-Go Branch}
\label{sec:res-abort}
We construct an empty-window gap jump requiring both clearance and landing
on an $8^\circ$ catch ramp. Because the point-mass model underpredicts
the measured range by ${\approx}39\%$ at 28\,m/s, the gap length is
calibrated per launch geometry. Fixed-state ladders yield 0/14 clean
landings for a 33\,m gap on the $8^\circ$ launch and 0/14 for a 54\,m gap
on the longer-range $12^\circ$ launch. The resulting
$v_{\min}\approx36.1$\,m/s exceeds
$v_{\mathrm{crit}}=11$\,m/s, so the certified speed set is empty.

With a 22\,m/s request, all three approach trials issue a no-go and stop
$5.03$--$5.04$\,m before the lip, meeting the 5\,m target. Together with the
nonempty-window gate experiment, these trials demonstrate both gate outcomes
by
selecting speed from a nonempty certified projection and refusing when that
projection is empty.

\subsection{Decisive Comparisons}
\label{sec:res-decisive}
This tier compares DART with TOBB and RW-PD under controlled airborne
disturbances and deployment-relevant approach chains, then isolates the
contributions of the gate and airborne control law.

\subsubsection{Air-Impulse Head-to-Head}
\label{sec:res-air}
Stressed conditions use the single-vehicle interleaved protocol and
same-law noise floors; the remaining conditions use the same-tick three-law
launch described in Section~\ref{sec:setup}.

Table~\ref{tab:headtohead} compares the interleaved Condition-1 experiments
using $N=30$ per law and angle plus a 30-jump DART replicate. RW-PD enters the
${\approx}180^\circ$ yaw-resonance mode at every angle and is significantly worse
than DART at $20^\circ$ ($p=6.5\times10^{-4}$), but not at $24^\circ$ or
$28^\circ$. DART and TOBB are statistically indistinguishable at $20^\circ$ and
$28^\circ$; TOBB outperforms DART at $24^\circ$ ($p=1.5\times10^{-5}$),
where DART itself partially enters the yaw-resonance mode. All laws complete
every condition without tumbling.

\begin{table}[!t]
\centering
\caption{Air-impulse head-to-head across takeoff angles under Condition~1,
with $N=30$ per law and angle. Each controller entry is median absolute
landing-pitch error ($^\circ$) [on-target/total] / median airborne yaw drift
($^\circ$); DART pool combines two 30-jump runs, and floor is their median
pitch-error gap ($^\circ$).}
\label{tab:headtohead}
\footnotesize
\setlength{\tabcolsep}{2.5pt}
\begin{tabular}{@{}lcccc@{}}
\toprule
$\alpha$ & DART pool & TOBB & RW-PD & floor \\
\midrule
$20^\circ$ & 3.7 [58/60] / 19 & 3.9 [29/30] / 20 & 17.5 [25/30] / 180 & 0.1 \\
$24^\circ$ & 8.7 [54/60] / 80 & 2.5 [29/30] / 23 & 9.4 [19/30] / 180 & 0.7 \\
$28^\circ$ & 13.1 [49/60] / 173 & 10.9 [23/30] / 180 & 19.8 [16/30] / 180 & 0.9 \\
\bottomrule
\end{tabular}
\end{table}

On mild entries, DART's fixed pitch gain shows no systematic advantage and
causes overshoot in 12/30 Condition-2 jumps; the per-flight latch removes
the cost of unnecessary roll engagement. Across 302 on-target landings,
peak-deceleration
bin medians rise from 14.1 to 15.1 to 19.0\,g as total attitude error crosses
5 and $15^\circ$, although the pooled rank correlation is weak
($r=0.11$, $p=0.065$).

\subsubsection{Approach Full-Chain Validation}
\label{sec:res-appr}
Approach mode builds the takeoff state through the ground run-up rather
than prescribing it at the lip. We compare the laws on straight and banked
run-ups, then separate airborne-law from gate contributions.

\begin{figure}[!t]
\centering
\includefig[width=\columnwidth]{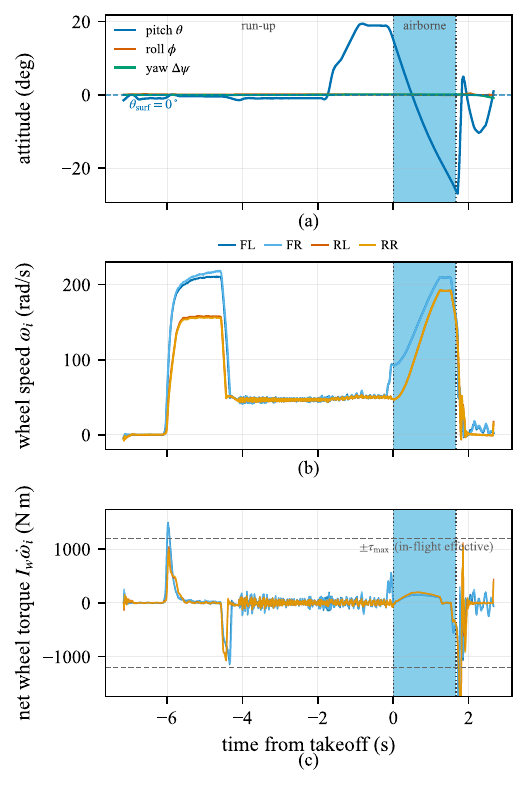}
\caption{Median-representative DART steep-lip jump corresponding to
Table~\ref{tab:appr}. (a)~Attitude, (b)~forward-positive wheel speeds, and
(c)~net per-wheel torque $I_w\dot\omega_i$, smoothed with a five-step moving average. Shading
marks the airborne interval.}
\label{fig:jumpts}
\end{figure}

Figure~\ref{fig:jumpts} traces one median-representative steep-lip jump.
The wheel-speed ramp near $-6$\,s is a slipping launch, so longitudinal slip
decouples wheel and vehicle speed. Launch wheelspin transiently exceeds the
1200-rpm sustained in-flight envelope while remaining below the 2105-rpm
drivetrain hard limit. In panel (c), $\pm\tau_{\max}$ constrains only the
airborne interval, where the motor is the sole torque source; launch and
touchdown excursions also contain tire--ground contact torque.
After launch wheelspin and speed hold, the abrupt end of lip support
imposes a $-31^\circ$/s
pitch rate; in-flight wheel spin-up relaxes it toward $-20^\circ$/s, and
the jump lands at the DART median error of $26.3^\circ$, while the baselines
land $12$--$16^\circ$ steeper. The remaining error reflects the cap the
momentum budget imposes and motivates pre-takeoff shaping.

\begin{table}[!t]
\centering
\caption{Approach-mode head-to-head on the $\alpha=20^\circ$ steep lip with
$N=30$. Safe-land counts use the crash-avoidance bounds of
Section~\ref{sec:setup}.}
\label{tab:appr}
\small
\setlength{\tabcolsep}{3.5pt}
\begin{tabular}{@{}lccc@{}}
\toprule
Controller & Safe-land & \(|\)pitch err\(|\) & \(|\)roll\(|\) \\
\midrule
DART  & 29/30 & 26.3\(^\circ\) & 0.1\(^\circ\) \\
TOBB  & 0/30  & 38.3\(^\circ\) & 0.1\(^\circ\) \\
RW-PD & 0/30  & 42.5\(^\circ\) & 0.4\(^\circ\) \\
\bottomrule
\end{tabular}
\end{table}

The straight-lane approach comparison in Table~\ref{tab:appr} is exactly
paired, with $N=30$ and $\alpha=20^\circ$. All controllers share the takeoff state
$(\theta_0,v_0,\omega_{\theta0})=(15.1^\circ,15.0\,\mathrm{m/s},
-31^\circ/\mathrm{s})$. DART reduces median pitch error to $26.3^\circ$ versus
$38.3^\circ$/$42.5^\circ$ (Wilcoxon $p\le2\times10^{-7}$) and achieves
29/30 safe landings versus 0/30 (McNemar $p=3.7\times10^{-9}$). With roll and yaw below
$1^\circ$, this isolates pitch-rate stress. In a gentler $\alpha=24^\circ$
companion experiment, all three controllers land safely, so kicker
severity, not angle alone, sets the regime.

\begin{table}[!t]
\centering
\caption{Cross-slope sweep with same-tick three-law launch at
$\alpha=20^\circ$ and $\gamma\in\{4,8,12\}^\circ$, using $N=30$ per law.}
\label{tab:appr-gamma-sweep}
\small
\setlength{\tabcolsep}{3pt}
\begin{tabular}{@{}ccccccc@{}}
\toprule
$\gamma$ & Law & Safe-land & \(|\)pitch\(|\) & \(|\)roll\(|\) & max roll air & $p_{\mathrm{vs\,DART}}$ \\
\midrule
\multirow{3}{*}{$4^\circ$}
 & DART  & 29/30 & 0.8\(^\circ\) & 4.8\(^\circ\) & 5.2\(^\circ\) & --- \\
 & TOBB  & 28/30 & 4.4\(^\circ\) & 4.6\(^\circ\) & 7.2\(^\circ\) & $<0.001$ \\
 & RW-PD & 28/30 & 3.6\(^\circ\) & 4.5\(^\circ\) & 8.2\(^\circ\) & $<0.001$ \\
\midrule
\multirow{3}{*}{$8^\circ$}
 & DART  & 30/30 & 2.0\(^\circ\) & 10.9\(^\circ\) & 10.9\(^\circ\) & --- \\
 & TOBB  & 29/30 & 2.6\(^\circ\) & 10.6\(^\circ\) & 11.1\(^\circ\) & 0.023 \\
 & RW-PD & 29/30 & 3.1\(^\circ\) & 7.9\(^\circ\) & 10.1\(^\circ\) & 0.007 \\
\midrule
\multirow{3}{*}{$12^\circ$}
 & DART  & 30/30 & 1.9\(^\circ\) & 10.9\(^\circ\) & 15.4\(^\circ\) & --- \\
 & TOBB  & 22/30 & 14.8\(^\circ\) & 12.4\(^\circ\) & 24.3\(^\circ\) & $<0.001$ \\
 & RW-PD & 27/30 & 14.7\(^\circ\) & 11.4\(^\circ\) & 24.7\(^\circ\) & $<0.001$ \\
\bottomrule
\end{tabular}
\end{table}

\begin{figure}[!t]
\centering
\includefig[width=\columnwidth]{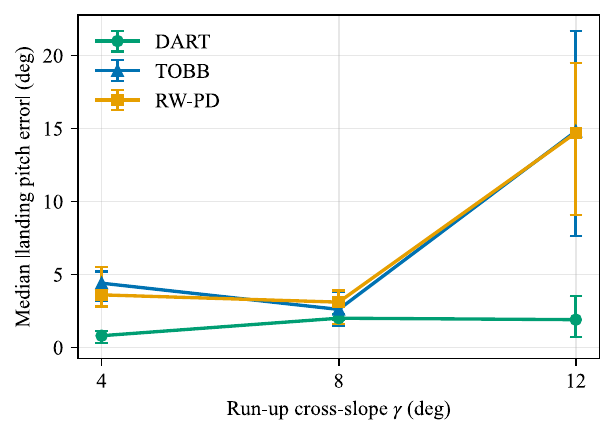}
\caption{Median absolute landing-pitch error versus run-up cross-slope
$\gamma$ under same-tick three-law launch with $N=30$ paired jumps per law
and cross-slope. Error bars span the interquartile range. DART stays near
$2^\circ$ across the sweep while both baselines
degrade sharply at $\gamma{=}12^\circ$, consistent with dual-axis allocation
once sustained roll-authority demand exceeds the pitch-only budget.}
\label{fig:gamma-sweep}
\end{figure}

We introduce lateral disturbances using three identical ramp copies and
run paired experiments at $\gamma\in\{4,8,12\}^\circ$ and $\alpha=20^\circ$
using $N=30$ per law. Table~\ref{tab:appr-gamma-sweep} reports all landed
trials. DART stays at or below $2^\circ$ median pitch error and outperforms
both baselines in Fig.~\ref{fig:gamma-sweep} at
every $\gamma$ ($p<0.05$); at $12^\circ$, TOBB's safe-landing count falls
to 22/30
and DART also reduces peak in-air roll ($p<10^{-4}$). A gate-off
$\gamma=8^\circ$ replicate remains statistically indistinguishable from the
gated configuration ($p=0.30$), attributing the pitch advantage primarily to
the airborne law.

\begin{table*}[!tp]
\centering
\caption{Geometric generalization under same-tick three-law launch with
$N=27$--$30$ valid paired jumps per geometry. Controller entries are median $|$pitch
error$|$ ($^\circ$); advantage $=$ best baseline $-$ DART.}
\label{tab:geomgen}
\begin{tabular}{@{}llccccccc@{}}
\toprule
Geometry & Mode & \(\theta_0\) (\(^\circ\)) & \(\omega_{\theta0}\) (\(^\circ\)/s) & Airtime (s) & DART (\(^\circ\)) & TOBB (\(^\circ\)) & RW-PD (\(^\circ\)) & advantage (\(^\circ\)) \\
\midrule
\(\alpha{=}10^\circ\) shallow & approach & 3.9 & \(-33\) & 1.2 & 24.0 & 29.6 & 34.0 & \(+5.6\) \\
\(\alpha{=}14^\circ\) nominal & approach & 7.3 & \(-34\) & 1.3 & 23.9 & 31.1 & 35.2 & \(+7.2\) \\
\(\alpha{=}22^\circ\) steep   & approach & 16.4 & \(-28.5\) & 1.7 & 22.1 & 34.0 & 36.9 & \(+11.9\) \\
\(\alpha{=}14^\circ\) tall    & approach & 7.0 & \(-36\) & 0.7 & 11.0 & 13.6 & 25.3 & \(+2.6\) \\
low launch      & air-impulse & 8.7 & \(\approx1\) & 1.3 & 1.1 & 6.2 & 7.4 & \(+5.2\) \\
nominal launch  & air-impulse & 8.7 & \(\approx1\) & 1.5 & 1.1 & 5.3 & 6.6 & \(+4.3\) \\
tall launch     & air-impulse & 8.7 & \(\approx1\) & 1.9 & 0.7 & 4.0 & 5.1 & \(+3.3\) \\
\bottomrule
\end{tabular}
\end{table*}

\subsubsection{Matched-Geometry Attribution}
Matching landing geometry, pitch, and speed while changing only takeoff
pitch-rate from $+12^\circ$/s under air impulse to $-31^\circ$/s under approach
increases DART error from $5.5^\circ$ to $23.3^\circ$, but degrades the
baselines from $7.7$--$13.0^\circ$ to $32$--$36^\circ$. The matched pair thus
attributes the larger approach separation to nose-down rate rather than
landing geometry.

\subsubsection{End-to-End Cross-Phase Chain on a Single Vehicle}
\label{sec:res-e2e}
The gate-off/on steep-lip experiment appears in Fig.~\ref{fig:e2e}. It uses
$N=30$ per configuration and
$v_{\mathrm{req}}\sim\mathcal{U}(21,23)$\,m/s. Gating reduces median
takeoff speed from 22.3 to 11.9\,m/s and touchdown speed by 36\%
($25.8\to16.5$\,m/s, $p=3\times10^{-11}$). It also raises on-target
landings from 0/30 to 30/30 ($p=1.7\times10^{-17}$). Both the gate-off and gate-on
configurations remain over the
pitch-rate budget, so this result establishes energy shaping and recovery
of on-target landings rather than full attitude nulling. Landing pitch
error, however, is \emph{larger} with the gate on: the medians are
$32.8^\circ$ and $18.1^\circ$, respectively. Traversing the lip more slowly prolongs the abrupt
front-axle unloading at the lip edge and hence its nose-down rotation,
roughly doubling the median takeoff pitch rate $|\omega_{\theta0}|$ from
$21$ to $43^\circ$/s while shortening the airtime available for correction
from $2.11$ to $1.46$\,s. The 47\% takeoff-speed
reduction exceeds the 36\% touchdown-speed reduction because
touchdown speed combines the braked horizontal component with a vertical
component set mainly by the fall geometry. Across the same experiments, the
median vertical touchdown speed falls only from $13.7$ to $11.8$\,m/s while
the flight time shortens from $2.11$ to $1.46$\,s; braking therefore acts
almost entirely on the horizontal component. The $36\%$ figure is accordingly a
property of this experiment's calibration, which is the ratio of the requested
approach speed to the certified $v_{\mathrm{crit}}$ of this geometry, rather than a
general constant; what transfers across geometries is the mechanism of
reducing the horizontal component to the certified window. The
one\hyb dimensional regulator certifies only the speed axis, leaving this
speed--rate coupling unmanaged and motivating a full-state $\Fset$ membership
check that would also gate on the predicted rate but is not evaluated here.

\looseness=-1
A companion experiment replays the same gate-off/on protocol on a curved-lip
kicker with a 35\,m arc at the same $\alpha=20^\circ$ and a
$\beta=-12^\circ$ valley landing. It uses $N=30$ per configuration to test
whether these findings are tied to the tabletop lip. Both principal effects transfer: gating reduces median
takeoff speed from $22.7$ to $13.9$\,m/s and touchdown speed by $33\%$
($21.7\to14.6$\,m/s, $p{=}1.4\times10^{-5}$), and the slower lip traversal
again amplifies the median takeoff pitch rate $|\omega_{\theta0}|$ from $21$ to
$40^\circ$/s while shortening the airtime from $2.4$ to $1.1$\,s. The
outcome of that coupling, however, reverses: under the gentler curved-lip
unloading the amplified takeoff rate remains within the in-flight control
authority at the certified speed, so the median landing-pitch error
improves from $11.4^\circ$ to $2.0^\circ$ ($p{=}1.1\times10^{-6}$) instead
of degrading, and both configurations stay within the crash-avoidance
bounds with 30/30 on-target landings. This outcome saturates on the
curved-lip geometry and therefore cannot discriminate between
configurations. The paired experiments thus separate the mechanism from
its consequence: takeoff-rate amplification is a construction-independent
cost of slower lip traversal, while its effect on
the landing is decided by whether the resulting takeoff state retains
sufficient in-flight authority, not by disturbance size alone.

\subsubsection{Budget-Interior Takeoff Control}
\label{sec:res-budgetin}
A complementary experiment places the takeoff \emph{inside} the budget using
an air-impulse launch at the same $\alpha{=}20^\circ$ severity, calibrated by a
small wheel-reaction pitch-rate kick to a median takeoff rate of $-9^\circ$/s,
within the $9$--$13^\circ$/s nose-up wheel-momentum budget. The same-tick
three-law experiment uses $N=30$ per law. All three laws collapse onto the same median
landing-pitch error of $2.7^\circ$ with 30/30 safe landings each, so the
cross-law advantage vanishes where wheel authority suffices for any law.
Together with the over-budget steep-lip experiments, this closes the two-sided
prediction of Theorem~\ref{thm:budget}. DART separates from the baselines
only when the takeoff pitch rate exceeds the certified budget.

\subsubsection{Per-Flight Roll-Latch Ablation}
\label{sec:res-adaptive}
\looseness=-1
Paired latch ablations use $N=30$ per condition and confirm the predicted
conditional allocation. On flat approach, DART reduces median pitch error from
$4.0^\circ$ with always-on roll to $1.3^\circ$
($p=6\times10^{-5}$); on the $1.2\times$ platform, from $4.1^\circ$ to
$2.8^\circ$ ($p=7\times10^{-8}$). Independent replicates confirm both
effects ($p\le1.4\times10^{-4}$); a lane-swap control confirms the same
direction ($p=2.5\times10^{-5}$). On the stressed $28^\circ$ and
yaw-resonance conditions, latch and always-on differences are unresolved, so
the latch improves the benign and heavy-mass conditions at no measured cost on
the stressed ones.

\subsubsection{Control-Law Ablations and Baseline Fairness}
At the stressed injection condition, dual/pitch-only/roll-only differences
remain
within the $3.3^\circ$ same-law noise floor; evidence for dual allocation
instead comes from the latch ablation and the $12^\circ$ banked runs.
Rate tracking and naive angle-PD have similar medians ($p=0.14$), but rate
tracking reduces worst-case overshoot from $-25.8^\circ$ to $-20.8^\circ$.
Across the tuning grid, DART holds $2.2^\circ$ while both tuned baselines
remain above $11.5^\circ$. The closed-form laws cost
$1.15$--$1.31\,\mu$s/step, over $7600\times$ faster than the 100\,Hz period.

\subsection{Generalization and Robustness}
\label{sec:res-generalization}
This tier varies one generalization axis at a time around the reference
configuration, considering geometry, controller gains, platform parameters
and mass, run-up surface, and takeoff-state jitter or actuator latency.

\subsubsection{Geometric Generalization}
\label{sec:res-geomgen}
Each law supplies $N=27$--$30$ valid paired jumps across seven
low-disturbance geometries. DART's advantage is positive in all seven, ranging from $+2.6$ to
$+11.9^\circ$. Table~\ref{tab:geomgen} and Fig.~\ref{fig:geomgen} show two
monotone trends: the advantage grows with takeoff severity and shrinks with
airtime. The smallest margin of $+2.6^\circ$ occurs on the tall approach
ramp. DART safe-landing counts range from 22/28 to 30/30 across the seven
geometries.

\begin{figure*}[!t]
\centering
\begin{minipage}[t]{0.49\textwidth}
\centering
\includefig[width=\linewidth]{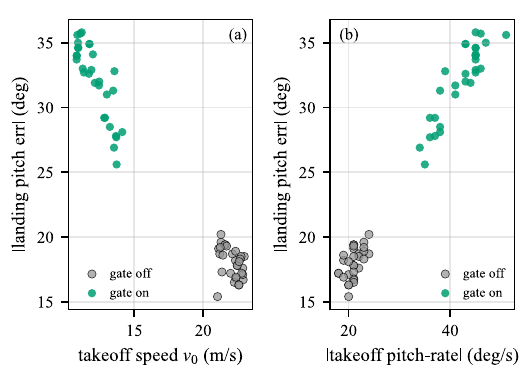}
\caption{End-to-end gate comparison on a single vehicle at the
$\alpha=20^\circ$ steep lip, with
$v_\mathrm{req}\sim\mathcal{U}(21,23)$\,m/s and $N=30$ per configuration.
Landing-pitch error versus (a)~takeoff speed
and (b)~takeoff pitch-rate; both configurations remain over budget.}
\label{fig:e2e}
\end{minipage}\hfill
\begin{minipage}[t]{0.49\textwidth}
\centering
\includefig[width=\linewidth]{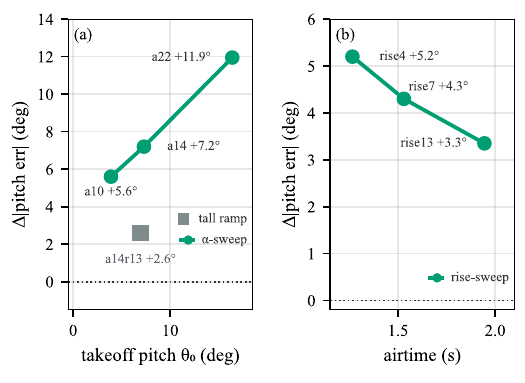}
\caption{Two monotone laws of the advantage from same-tick three-law
head-to-head with $N=27$--$30$ per geometry. (a)~Sweeping approach angle,
the advantage grows with takeoff severity $\theta_0$; the tall ramp, shown
as a gray square, remains positive
below the trend. (b)~Sweeping air-impulse launch height at fixed $\theta_0$,
the advantage shrinks monotonically with airtime.}
\label{fig:geomgen}
\end{minipage}
\end{figure*}

\subsubsection{Gain and Platform Robustness}
\label{sec:res-platform}
Table~\ref{tab:gain} sweeps each of DART's three gains over
\(0.5\times\)/\(1\times\)/\(2\times\) and keeps landing-pitch accuracy within
\(0.45\)--\(2.2^\circ\) across the grid, a spread below \(1.8^\circ\). These sweeps
assess robustness within DART only; cross-law comparisons are reported in
Section~\ref{sec:res-air}. As \(K_\phi\) increases, roll decreases
monotonically from \(23.5^\circ\) to \(10.4^\circ\). Raising \(K_\theta\)
borrows front-axle authority from roll, increasing it from
\(15.1^\circ\) to \(18.0^\circ\), directly
illustrating the co-allocation of
pitch and roll authority. Safe-landing ranges from 27/30 to 30/30 across
the seven settings. Here $K_d$ is the rate-loop torque gain, distinct from the platform torque bound
$\tau_{\max}$.

\begin{table}[!t]
\centering
\caption{Controller gain sensitivity with $N=30$ per setting; entries are
medians in degrees.}
\label{tab:gain}
\begin{tabular}{@{}lcc@{}}
\toprule
Gain swept & \(|\)pitch err\(|\) (0.5/1/2\(\times\)) & \(|\)roll\(|\) (0.5/1/2\(\times\)) \\
\midrule
\(K_\theta\) (pitch-rate) & 1.65 / \textbf{2.2} / 0.45 & 15.1 / 16.1 / 18.0 \\
\(K_\phi\) (roll-steer)   & 1.1 / \textbf{2.2} / 1.95 & 23.5 / 16.1 / 10.4 \\
\(K_d\) (torque)          & 2.2 / \textbf{2.2} / 2.2 & 16.1 / 16.1 / 15.8 \\
\bottomrule
\end{tabular}
\end{table}

Across motor torque, tire pressure, CoM, and pitch inertia, DART stays
within a $2^\circ$ spread and lands 29--30/30 safely, as
Table~\ref{tab:platform} shows. Insensitivity to motor torque confirms that the
momentum budget, not peak torque, sets the limit. In the always-on
mass ladder, pitch error rises from $2.1$ to $5.1$ to $8.8^\circ$ as mass
scales from $0.8$ to $1.0$ to $1.2\times$; the trend has $p<10^{-5}$.
Enabling the latch removes
the heavy-mass penalty. The torque variants scale the effective per-wheel
bound $\tau_{\max}$ by $0.6\times$ and $1.4\times$.

\begin{table}[!t]
\centering
\caption{Platform robustness with $N=30$ per configuration. DART and
best-baseline entries are median absolute landing-pitch errors; advantage is
best baseline minus DART. All values are in degrees. Slash-separated values follow the
listed variant order; Safe-land is safe/total, with ``each'' applying to both
paired variants.}
\label{tab:platform}
\footnotesize
\setlength{\tabcolsep}{2pt}
\begin{tabular}{@{}lcccc@{}}
\toprule
Variant & DART & best base. & adv. & Safe-land \\
\midrule
Nominal                 & 2.2\(^\circ\)  & 6.2\(^\circ\)  & \(+4.0^\circ\) & 30/30 \\
Torque 0.6\(\times\)/1.4\(\times\) & 0.9/2.2 & 6.35/5.85 & \(+5.4\)/\(+3.6\) & 30/30 each \\
Tire pressure low/high & 2.2/2.25    & 6.15/6.05    & \(+4.0\)/\(+3.8\) & 30/30 each \\
CoM forward            & 2.3         & 6.05         & \(+3.8\)      & 30/30 \\
Pitch-inertia high/low & 1.2/1.9     & 7.0/6.65     & \(+5.8\)/\(+4.8\) & 29/30 each \\
\bottomrule
\end{tabular}
\end{table}

\subsubsection{Run-Up Surface Robustness}
\label{sec:res-ground}
Under the approach protocol at $\alpha=20^\circ$, each always-on-roll
configuration uses $N=30$. Replacing the asphalt run-up with gravel or dirt
leaves landing accuracy statistically unchanged for all three laws. DART's
median spans $3.1$--$5.1^\circ$, with $29$--$30/30$ on-target per surface.
Takeoff-state
shaping by
the reachability gate absorbs the surface-dependent traction differences, so
the airborne comparison is not confounded by run-up surface.

\looseness=-1
A companion landing-surface experiment repeats the $\alpha{=}20^\circ$
air-impulse head-to-head with an identical curved lip. Its takeoff state is
$\theta_0=11.6^\circ$ and $\omega_{\theta0}=+12^\circ$/s, while the
landing valley is flattened from $\beta{=}{-}12^\circ$ to
$\beta{=}{-}2^\circ$. All three laws then land 30/30 with medians of
$1.0$--$1.7^\circ$, versus $3.7^\circ$ for DART and $17.5^\circ$ for RW-PD on
the $-12^\circ$ valley: a near-level target renders attitude tracking
undemanding for every law, and the cross-law separation vanishes. The
valley slope, not the landing surface itself, is therefore what keeps
attitude tracking demanding enough to separate the laws; the flat-floor
landings in Tables~\ref{tab:geomgen}--\ref{tab:platform} replicate these
conclusions on level terrain.

\begin{table*}[!tp]
\centering
\caption{Qualitative disturbance-regime map.}
\label{tab:regime}
\begin{tabular}{@{}>{\raggedright\arraybackslash}p{0.115\textwidth}>{\raggedright\arraybackslash}p{0.15\textwidth}>{\raggedright\arraybackslash}p{0.33\textwidth}>{\raggedright\arraybackslash}p{0.27\textwidth}@{}}
\toprule
Regime & Dominant stress & Observed airborne outcome & Primary evidence \\
\midrule
A: pitch-rate & large sustained $|\omega_{\theta0}|$ & DART $\gg$ RW-PD in pitch & appr.\ $\alpha{=}20^\circ$ steep lip \\
\addlinespace
B: dual-axis & large roll and pitch-rate & DART $\gg$ RW-PD at $20^\circ$; TOBB leads at $24^\circ$; DART/TOBB converge at $28^\circ$ & air $\alpha{\in}\{20,24,28\}^\circ$; Tab.~\ref{tab:headtohead} \\
\addlinespace
C: banked run-up & roll from $\gamma$ & DART $>$ both baselines in pitch at every $\gamma$; gap largest at $\gamma{=}12^\circ$ & appr.\ $\gamma{\in}\{4,8,12\}^\circ$; Tab.~\ref{tab:appr-gamma-sweep} \\
\addlinespace
D: benign & small roll, mild pitch-rate & latch improves flat/heavy conditions; all laws are safe on the gentle valley & flat appr.\ at $\alpha{=}20^\circ$ / gentle valley at $\alpha{=}24^\circ$; Sec.~\ref{sec:res-adaptive} \\
\midrule
\multicolumn{4}{@{}l}{\textit{Cross-phase (all regimes): reachability gate independent of the airborne row}} \\
Gate / shaping & $v_0$ vs.\ landing geometry & $-36\%$ touchdown speed; $0/30\to30/30$ on-target & Section~\ref{sec:res-e2e} \\
\bottomrule
\end{tabular}
\end{table*}

\subsubsection{Takeoff Jitter and Actuator Latency}
\label{sec:res-jitterlat}
Sweeping jitter over $\sigma\in\{0.25,1,2,3,5\}$ and latency over
$\{0,10,20,40\}$\,ms with $N=30$ per setting changes pitch medians
moderately but exposes a low-dispersion yaw resonance. In the interleaved latency sweep,
RW-PD reaches ${\approx}180^\circ$ yaw and 17--21/30 on-target landings,
TOBB 29--30/30, and DART lies between them. Exposure therefore tracks
steering usage rather than latency alone; the per-flight latch removes this
exposure in the deployment-relevant flat and heavy-mass conditions.

\subsection{Regime Synthesis and Operating Boundaries}
\label{sec:res-regime}
The preceding experiments support the regime synthesis in
Table~\ref{tab:regime}, which maps takeoff severity to DART-to-baseline
separation across the cross-phase chain and the airborne law.

Regime~D matches Theorem~\ref{thm:frontier} because, when roll does not bind,
the latch restores pitch-only allocation and avoids unnecessary yaw coupling.

\subsubsection{Operating Boundaries}
\label{sec:res-limits}
Four boundaries delimit the tested range: (i)~beyond
$B_\uparrow\approx13^\circ$/s, every law retains large steep-lip residual;
(ii)~long airtime allows the baselines to approach DART's accuracy;
(iii)~mild entries can expose fixed-gain overshoot; and
(iv)~steer-vectored roll corrections can accumulate same-sign yaw impulses
that feed the yaw-resonance mode. These require, respectively, pre-takeoff shaping, no special
intervention, gain scheduling, and conditional roll allocation. No abrupt
tumbling transition appears within the tested envelope.

\section{Discussion and Limitations}
\label{sec:limits}

\phantomsection\label{sec:practicality}%
Among the practical constraints on deploying DART, the most fundamental is
vehicle scale. At 1383\,kg the wheels can supply only \(9\)--\(13^\circ\)/s of nose-up
body pitch rate at typical takeoff wheel speeds. The reverse-inclusive
braking direction affords roughly twice as much, as
Theorem~\ref{thm:budget} shows, while the drivetrain hard limit raises the
nose-up ceiling to no more than $16$--$18^\circ$/s. Because the budget
scales as \(B\propto1/m\) under fixed wheel inertias, this authority
tightens further as vehicles grow heavier. Torque sweeps leave landing
accuracy nearly unchanged, which confirms that stored angular momentum, not
peak actuation, sets the limit. Airborne control is therefore
best understood as a trimming layer for takeoff states already near
\(\Fset\); disturbances beyond the budget must be handled before the wheels
leave the ground. The pitch budget must be re-identified for each vehicle
class, and the steer-vectored roll channel carries a limitation of its own,
yaw drift, so hardware transfer entails budget identification, calibration
of \(\mathcal B_\rho\), and staged validation on reduced-scale platforms.

\phantomsection\label{sec:honest-boundaries}%
Even inside the certified envelope, three residual costs shape deployment.
Spending wheel momentum in flight changes the touchdown wheel speed,
coupling attitude control to the landing itself; a fixed pitch gain tuned
for stressed entries can overshoot mild ones; and steer-vectored roll
correction leaks into the unactuated yaw axis. Hardware transfer therefore
requires an actuator--landing interaction model, state-dependent pitch gain
scheduling, and yaw-loop phase compensation. The per-flight roll latch
already suppresses unnecessary roll actuation on benign entries and heavy
platforms.

\phantomsection\label{sec:terrain-req}%
The framework's terrain requirements are correspondingly layered. Correcting
the pitch rate needs no terrain information; steering the attitude
toward its target requires only the mean landing slope over a vehicle
footprint; and certifying feasibility requires the landing-zone location and
drop at planning resolution, but never in-flight surface friction. As the
landing estimate becomes less certain, \(\Fset\) erodes and the gate's
decision shifts conservatively toward no-go.

\phantomsection\label{sec:uncertainty}%
A symmetric question arises on the takeoff side. On hardware,
state-estimation error, actuator latency, tire slip, and suspension scatter
together populate the uncertainty ball \(\mathcal B_\rho\), and the eroded
set \(\Fset\ominus\mathcal B_\rho\) supplies the margin: if it is
empty, the vehicle must shape earlier or abort. Estimating \(\rho\) and
validating this margin are the primary sim-to-real tasks.

\phantomsection\label{sec:future-expts}%
The next validation steps are online full-state \(\Fset\) membership during
approach, since the evaluated regulator certifies the speed axis only, and
yaw-loop compensation for the residual resonance.
Reduced-scale four-wheel-drive experiments following XCar~\cite{xcar2025iwd}
are the natural next step, preceding full-vehicle identification.

\section{Conclusion}
\label{sec:conclusion}

We presented DART, a cross-phase attitude-control framework for high-speed
vehicle jumps with distributed electric drive, organized around the
angular-momentum budget that caps recoverable in-flight pitch rate. At
1383\,kg, the nose-up budget is \(\approx 9\)--\(13^\circ\)/s and
tightens with mass. The
framework combines backward-reachability pre-takeoff speed shaping with a
regime-conditional dual-axis in-flight law, validated by paired
deterministic simulation with approach-mode run-ups.

End-to-end, pre-takeoff speed shaping reduced touchdown speed by 36\% and
raised on-target landings from 0/30 to 30/30 in the evaluated steep-lip
scenario, while the gate refused the jump and stopped the vehicle
before it reached the lip when the certified speed window was empty. Under
stressed entries, DART satisfies the crash-avoidance landing bounds where
RW-PD and TOBB fail (29/30 vs.\ 0/30 on steep-lip approach) and holds the
median pitch error at or below $2^\circ$ over banked run-ups, with the widest
separation from the baselines at $\gamma{=}12^\circ$.

The airborne advantage is regime dependent. Engaged roll control serves
dual-axis disturbances, whereas the per-flight latch restores pitch-only
behavior on benign entries and limits yaw-leak exposure. RW-PD's continuous
counter-steer can feed the resonance mode. On the pure-attitude authority
ladder, pitch residual stays below \(\approx 3^\circ\). Peak-deceleration bin
medians increase with touchdown attitude error, although the pooled rank
correlation is weak. Rate errors exceeding the directional budgets
$B_\uparrow$ or $B_\downarrow$ by more than the terminal-window half-width
are provably unrecoverable in flight.

Future work will add takeoff-state gain scheduling and proceed from
reduced-scale hardware validation to full-vehicle identification of the
pitch-rate budget at each vehicle scale.

\section*{Acknowledgments}
During the preparation and technical implementation of this work, the
authors utilized a combination of generative artificial intelligence (AI)
systems in an integrated workflow. Specifically, Auto/Composer 2.5 (Cursor),
GLM 5.2, DeepSeek V3/V4, and Kimi K3 were used interactively to assist in developing,
debugging, and optimizing the algorithmic source code and data-processing
pipelines discussed in Section~\ref{sec:method}. These same systems
(GLM 5.2, DeepSeek V3/V4 and Kimi K3) were also co-employed to generate the analytical
visualization scripts for creating the experimental data charts in
Section~\ref{sec:results}. During the manuscript revision phase, GLM 5.2, DeepSeek V3/V4
and Kimi K3 refined English prose, enhanced language fluency, and polished
the text across all sections. In accordance with IEEE guidelines, all
AI-assisted engineering code, data visualizations, and textual outputs were
thoroughly reviewed, logically validated, and cross-checked against raw
experimental results by the human authors, who retain full academic and
ethical responsibility for the content and integrity of this publication.

\balance

\end{document}